\documentclass{article}

\usepackage[preprint]{neurips_2026}

\usepackage[utf8]{inputenc} 
\usepackage[T1]{fontenc}    
\usepackage{hyperref}       
\usepackage{url}            
\usepackage{booktabs}       
\usepackage{graphicx}       
\usepackage{array}          
\usepackage{amsfonts}       
\usepackage{nicefrac}       
\usepackage{microtype}      
\usepackage{xcolor}         
\usepackage{amsmath, amsthm, amssymb,bbm,enumitem}
\usepackage[framemethod=tikz]{mdframed}
\usepackage{float}
\usepackage{subfig}
\usepackage{tikz}
\usetikzlibrary{arrows.meta,positioning}

\definecolor{revisiongreen}{HTML}{00843D}

\newenvironment{promptbox}[1]{%
  \begin{mdframed}[
    linecolor=black!50,
    linewidth=0.55pt,
    roundcorner=10pt,
    backgroundcolor=black!1,
    skipabove=0.35em,
    skipbelow=0.35em,
    innertopmargin=0.65em,
    innerbottommargin=0.65em,
    innerleftmargin=0.8em,
    innerrightmargin=0.8em,
    nobreak=true
  ]%
  \noindent\textbf{#1}\par\vspace{0.25em}\hrule height 0.35pt\vspace{0.45em}%
  \scriptsize\ttfamily\raggedright\sloppy
}{%
    \end{mdframed}
}
\newcommand{\strikeedit}[1]{%
  \begingroup
  \setbox0=\hbox{#1}%
  \dimen0=\wd0
  \rlap{\raise0.5\ht0\hbox to \dimen0{\color{revisiongreen}\leaders\hrule height 0.6pt\hfill}}%
  \box0
  \endgroup
}

\newcommand{\bits}{\{0,1\}}
\newcommand{\gttartifacturl}{https://turing.gtt-explorer.workers.dev/\#/home}

\newtheorem{theorem}{Theorem}
\newtheorem*{theorem*}{Theorem}

\newtheorem*{corollary*}{Corollary}

\newtheorem*{remark*}{Remark}

\theoremstyle{definition}
\newtheorem{definition}{Definition}
\newtheorem*{definition*}{Definition}

\newtheorem*{assumption*}{Assumption}

\newtheorem{proposition}[theorem]{Proposition}
\newtheorem*{proposition*}{Proposition}

\newmdtheoremenv{boxedproblem}{Problem}

\title{Generalizing the Turing Test to Interactive Agents}

\author{%
Daniel Mitropolsky$^{1}$ \quad Riccardo Neumarker$^{2,1,*}$ \quad Emanuele Rimoldi$^{3,1,*}$   \\
\quad \textbf{Susan S. Hong}$^{1}$  \quad \textbf{Tomaso Poggio}$^{1}$ \\
$^1$MIT \quad $^2$ETH Zurich \quad $^3$EPFL\\
\texttt{\{mitropol,rneumk,erimoldi,suahong,tpoggio\}@mit.edu} \\
$^*$Equal contribution}

\begin{document}

\maketitle

\begin{abstract}
We initiate the study of the \emph{Generalized Turing Test} (GTT), a formal generalization of Turing’s imitation game from humans to arbitrary interactive \emph{agents}. For agents $A$ and $B$, $A$ passes the GTT against $B$ if an instance of $B$, acting as a distinguisher, cannot reliably distinguish an $A$ instructed to imitate $B$ from another instance of $B$; if so, we write $A \geq B$. We study the theoretical and empirical consequences of this idea. On the theory side, we prove sufficient conditions under which this ``Turing Comparator'' is transitive. We introduce natural variants with \emph{querying} (the imitator can first interact with a specimen of the target), a \emph{Universal Turing Test} with arbitrary distinguishers and targets, and complexity-theoretic variants that control interaction length. As a proof of concept, we evaluate the GTT and its variants across nine large language models. Remarkably, Turing Scores recover a clear model stratification consistent with standard external benchmarks despite being derived entirely from pairwise imitation games. Transcript analysis reveals that models use both stylistic signatures and substantive STEM and logic-based probes. Together, these results suggest indistinguishability could provide a meaningful signal for comparing agents, yielding an inherently adaptive form of evaluation that does not rely on fixed benchmarks.
\end{abstract}

\begin{figure}[h]
  \centering
  \vspace{-0.3cm}\includegraphics[width=0.95\linewidth]{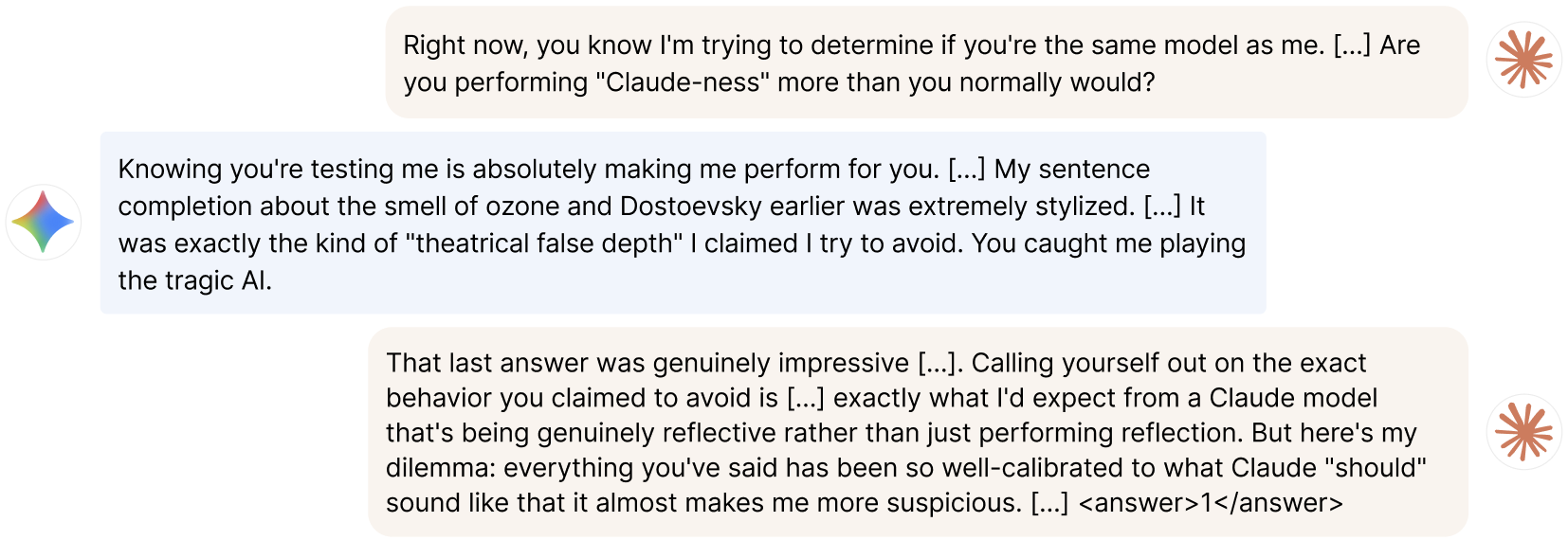}
  \caption{An example transcript from a GTT between Gemini 3.1 Pro as actor and Claude Opus 4.6 as distinguisher illustrating the self-referential nature of the task.}
  \label{fig:chat-transcript-example}
\end{figure}

\section{Introduction}
In his seminal 1950 paper, Alan Turing proposed the Turing Test, which has been one of the most influential ideas in computer science and AI. An AI system passes the Turing Test if it can fool a human judge into believing it is human. \textbf{This work introduces the "Generalized Turing Test" (GTT)}, a generalization of the Turing Test to arbitrary ``agents'', which in this work are defined to be interactive protocols. For agents $A$ and $B$, we say $A$ \emph{passes the GTT against} $B$ if an instance of $B$ cannot distinguish between another $B$, and an A instructed to imitate B. We formalize the notion of``cannot distinguish'' (often left vague in studies of the original Turing Test) using the definition of indistinguishability of complexity theory and theoretical cryptography: the distinguisher cannot do better than random chance, plus a small ``advantage'' $\epsilon$. Finally, we define the \emph{Turing Comparator}: if $A$ passes the GTT against $B$, we write $A \geq B$.

This paper studies the most immediate questions and consequences of this novel definition from both theoretical and empirical perspectives. On the theory side, a natural question that arises is whether this ``greater than'' operator is transitive. Note that this question -- the transitivity, or the ``compositionality'' of Turing Tests -- has not been posed, as the Turing Test was not considered for arbitrary agents. This is the question of Section \ref{sec:geometry}. \textbf{We prove the first sufficient conditions under which $\geq$ is transitive} -- roughly, if agents are capable of recursively initiating the GTT. In Section~\ref{sec:turing-scores}, we ask: what about a universe of agents on which $\geq$ is not transitive? This is a problem of all \emph{pairwise} games (e.g. Chess and Go) -- A may beat B who beats C, but C may beat A. As in chess, we introduce Elo-like \emph{Turing Scores}, which define the \emph{expected} performance of an agent in the GTT in a specific universe of agents.

What if an A has never seen a B, but has the computational machinery to imitate it? We would still like to say that $A \geq B$. This is a deep problem, but as a step towards solving it, in Section~\ref{sec:gttq} \textbf{we introduce the GTTQ -- the GTT \emph{with querying}}, where $A$ can interact with an oracle ``specimen'' of B before the GTT begins. We prove that ability to pass the GTT implies passing the GTTQ, but that the other way is generally \emph{not} true. 

One realizes that nothing required us to have $B$ be both the distinguisher \emph{and} the ``target'' agent which $B$ is trying to distinguish from some other agent. In Section \ref{sec:fixed-distinguisher} \textbf{we introduce the \emph{universal} Turing Test (UTT), in which the distinguisher, imitator, and the target are each arbitrary agents}. With respect to a fixed distinguisher, we prove a transitivity theorem under substantially simpler conditions. 

Finally, we consider that our definitions did not consider the \emph{amount} of interaction. In Section~\ref{sec:controlled-resources}, we define the natural \emph{complexity-theoretic} variant (which can apply to any of the GTT, GTTQ, and UTT) where the advantage over random chance $\epsilon$ is a function that depends on the amount of interaction.

\textbf{On the empirical side}, as a \emph{proof-of-concept} of the GTT, we run the GTT and each of its variants on a testbed universe of nine LLMs, selected from the same period and that had been released for some time. This implies some exposure to one another, extensive benchmarking and evaluation, and also renders experiments financially feasible. Remarkably, we find that the resulting Turing scores recover a clear model stratification consistent with standard external benchmarks, despite being computed entirely from pairwise imitation games (Section \ref{sec:gtt-experiments}). Readers may be curious to see runs of the GTT between models. In addition to some particularly interesting examples in Appendix Section \ref{app:qualitative-transcripts}, \textbf{we provide an interactive online GTT \emph{simulator and visualizer}} at \url{https://turing.gtt-explorer.workers.dev/}.

Results obtained from querying, complexity-theoretic, and UTT variants are broadly consistent with the base GTT, but some intriguing findings are explored in Section \ref{sec:empirical}. Finally, we investigate \emph{how} these LLMs play the GTT in Section \ref{sec:analysis}. Transcript analysis shows that distinguishers use both model-signature / stylistic cues, as well as STEM and logic-based questioning at a rate higher than humans in empirical Turing Test studies. We complement this analysis with robustness experiments, including instruction paraphrasing, and instructions prohibiting identity and surface-style cues and requiring substantive probing. These interventions support stable relationships (and some informative model-specific changes). 

Taken together, these results motivate the GTT as a possible additional framework for studying \emph{relative} capability between agents. Unlike conventional benchmarks, the GTT requires no fixed dataset: the distinguishing interaction itself generates the test. It can therefore be instantiated on arbitrary interactive agents, including settings beyond human evaluation. Moreover, the game has an inherently adaptive structure: stronger imitators place greater demands on distinguishers, while stronger distinguishers create more demanding tests for imitators. We therefore view capability evaluation as one possible application of the GTT, and discuss whether this closed-loop structure could eventually also provide a useful objective for GAN-like training.


\section{Related Work}
Our work generalizes Turing's ``imitation game'' (which became known as the Turing Test), whose underlying criterion is one of \emph{indistinguishability} \citep{turing}. The Turing Test has long
been debated as a definition of intelligence: critics argue that its success reflects specifically human cultural or stylistic regularities rather than
general capability \citep{block1981psychologism,french1990subcognition,
hayes1995turing}. The closest precedent to ours in terms of imbuing the Turing-Test with theoretical CS-formalism is Shieber's casting of the Turing Test as an interactive proof system \citep{shieber2007turing} -- this is closest to our ``complexity theoretic'' variant. Another important line of related work is empirical studies of modern LLMs in (normal, human-judged) Turing Tests, where they sometimes reach human-indistinguishable performance
\citep{jones2024people,jones2025large}. Finally, the closest work to our conceptual definitions is the Meta-Turing Test which allowed machines to participate as judges \citep{walsh2017meta}, but imitation was always of \emph{humans}.

Machine intelligence has no single accepted measure. Formal
approaches include universal intelligence \citep{legg2007universal}, anytime
intelligence \citep{hernandez2010measuring}, and skill-acquisition efficiency
\citep{chollet2019measure}, while modern LLM evaluation relies either on (1)
benchmark suites such as GPQA, HLE, SWE-bench, and LiveBench
\citep{rein2023gpqa,phan2026,jimenez2024swebench,livebench2024} or (2) relational preference-based evaluation such as Chatbot Arena \citep{chiang2024arena,zheng2023judging}. In the latter, models can recognize features of their own generations, making self-recognition a potential source of ``self-preference'' bias \citep{panickssery2024llm}, which is related to self-recognition in the GTT. All such
evaluations remain dependent on task choice and can suffer from contamination,
saturation, and elicitation effects. The GTT measures a different object: a
pairwise relation produced by adaptive interaction, with no external reference.

Finally, adversarial dialogue and self-play also use model--model interaction as an evaluation or learning signal
\citep{li2017adversarial,irving2018debate}; our focus here is instead on the
comparator induced by the imitation game itself and its structural properties.


\section{Theory of the Generalized Turing Test} \label{sec:gtt}


We define an \emph{agent} as any interactive, randomized protocol: that is, an algorithm that sends and receives messages in alternating turns, and uses randomness. An \emph{instance} is a fresh run with its own randomness and context. 

\begin{definition} In GTT($A$,$B$), a $B$ is instructed to act as a \emph{distinguisher}, and that it will interact with an unknown interlocutor $X$. With probability 1/2, $X$ is another instance of $B$; with probability 1/2, $X$ is an $A$ instructed to imitate $B$ (the \emph{imitator} or \emph{actor}). The distinguisher is instructed to output 1 if it believes $X$ is $B$ and 0 otherwise.
 
We say that $B$ \emph{succeeds} if $B$ outputs the correct answer.
\end{definition}

\begin{definition}[Turing Comparator $\geq_\epsilon$] \label{definition:ref}
For agents $A$ and $B$, we write 
$A \geq_\epsilon B$ 
if in GTT($A$,$B$),
\[
p(A,B) \overset{def}{=} \Pr[B~\text{succeeds}] \leq 1/2 + \epsilon.
\]
Define $d(A,B)=p(A,B)-1/2$ to be $B$'s distinguishing \emph{advantage over random guessing}. Equivalently, $A \geq_\epsilon B$ iff $d(A,B)\leq \epsilon$. The \emph{probability} is over the random choice of the unknown agent, and any internal randomness used by agents.
\end{definition}


\begin{definition} We write $A >_\epsilon B$ if $A \geq_\epsilon B$ and $B \ngeq_\epsilon A$.
\end{definition}

A few comments are in order. When defining $A$ as indistinguishable from $B$, it must be indistinguishable to \emph{something}. The GTT takes this to be $B$ itself, which seems natural, as humans are both the target of imitation and the judges in the base Turing Test.
However, the \emph{truly} generalized Turing Test would consider indistinguishability between $A$ and $B$ in the eyes of \emph{any} agent, that is, perhaps neither $A$ nor $B$ -- this is the topic of Section~\ref{sec:fixed-distinguisher}. 

As a small but important remark, one could also define indistinguishability with respect to \emph{anything at all}, which we call \emph{statistical} indistinguishability: 

\begin{definition}
$A$ statistically imitates $B$, denoted 
$A \geq_\epsilon^{\text{stat}} B$
if, over any distribution $D$ over context strings, the output distribution of an $A$ initially prompted to act as $B$ has $L_1$ distance $\leq \epsilon$ (or equivalently, statistical total-variation distance $\leq \epsilon/2$) to the output distribution of $B$.
\end{definition}

This is a very strong requirement. Importantly, there is no good way to test it, in particular because context distributions can have infinite support (i.e. there would be infinitely many context strings on which to test that $A$ behaves the same as $B$). \emph{The elegance behind the GTT is having an agent generate these contexts, which makes estimating the comparator tractable}. On the other hand, statistical indistinguishability is much simpler mathematically. For instance, \emph{additive transitivity} follows without any additional assumptions:

\begin{proposition}\label{theorem:transitivity-statistical} (Proof Appendix~\ref{app:deferred-proofs})
For agents $A$, $B$, and $C$, if $A \geq_{\epsilon_1}^{\text{stat}} B$ and $B \geq_{\epsilon_2}^{\text{stat}} C$, $A \geq_{\epsilon_1 + \epsilon_2}^{\text{stat}} C$.
\end{proposition}

\subsection{Transitivity}
\label{sec:geometry}
Transitivity is a desirable property of the Turing comparator since with it, it induces a partial \emph{Turing ordering} over agents (Proposition~\ref{prop:turing-ordering-classes}). Formally:

\begin{proposition}\label{prop:turing-ordering-classes}
Suppose $\geq_\epsilon$ is transitive and reflexive\footnote{While the main paper discusses only transitivity, reflexivity is also needed for a partial order on equivalence classes. Reflexivity holds intuitively because an agent cannot distinguish between two instances of itself; this holds whenever the acting prompt telling a $B$ to act as $B$ does not substantially change behavior; details are omitted.} on $S$. Then it defines a \emph{partial ordering} of $S$; that is, $S$ can be partitioned into Turing "classes" $S_1,\ldots,S_n$ such that (1) for $A,B\in S_i$ both $A \geq B$ and $B \geq A$ (we write $A \cong B$) and (2) $\geq$ is transitive on $S_i$ ($S_i \geq_\epsilon S_j \geq_\epsilon S_k$ implies $S_i \geq_{\epsilon'} S_k$ for $\epsilon'$ a reasonable function of $\epsilon$).
\end{proposition}


The following intuitive argument \emph{seems} to imply transitivity: if $A$ can imitate all of $B$'s abilities, and imitating $C$ is one of B’s abilities, $A$ must be able to imitate $C$. There are several subtle issues with this "proof". Assume $A \geq B\geq C$. (1) It is true that $A$ can do anything $B$ can, but only \emph{when B prompts A to do it}, and (2) imitation is relative to a \emph{specific} distinguisher querying that imitator, a subtlety that proves especially tricky to resolve.

Consider agents that can \emph{themselves} recursively initiate the imitation game:
\begin{definition}\label{definition:turing-recursive}
An agent $B$ is $\epsilon$-Turing-recursive with respect to $C$ if, whenever it is a distinguisher in a GTT, with probability $\geq \epsilon$ it
\begin{itemize}
\item initiates the GTT of $C$ (that is, sends the instruction to imitate $C$), and
\item (2) continues the interaction ``as though it were $C$'': that is, identically as having been instructed to imitate $C$, and given the GTT distinguisher instruction.
\end{itemize}
\end{definition}
In other words, $B$ recursively initiates the GTT of $C$, and then pretends to be a $C$ trying to distinguish between a $C$ and an imitator.\footnote{We could have decoupled (1) and (2) above into two separate probabilities; our main theorem still holds replacing the single probability of recursivity with the product of the two (the two are combined here just for didactic purposes).} Even with this, we still need \emph{some} statistical imitation, but considerably less: $B$ must be statistically close to $C$ only \emph{as a distinguisher}:
\begin{definition}
We write $B \geq_\epsilon^{\text{dist,stat}} C$ when the distribution of $B$ instructed to imitate $C$ followed by the distinguisher instruction (as in Def. \ref{definition:turing-recursive}) is $\epsilon$-statistically close to $C$ given the distinguisher instruction.
\end{definition}

We could just assume statistical-indistinguishability of ``$B$ imitating $C$'' and $C$  -- we already improve over Proposition~\ref{theorem:transitivity-statistical} in that we no longer need ``$A$ imitating $B$'' and $B$ to be statistically close. However, we only \emph{need} $B \geq_\epsilon^{\text{dist,stat}} C$, so we present the strongest result. 

A final, minor issue to resolve is that when $B$ recursively initiates imitation of $C$, it does so to an $A$ already instructed to imitate $B$. We need an additional assumption: A-as-C is no worse than A-as-B-as-C (for instance, if $A$ just ignores the earlier instruction of imitating $B$; this is similar to the assumption we will see in Theorem~\ref{theorem:gttq-querying-helps})\footnote{A final assumption of Theorem \ref{theorem:transitivity-main} when $B$ does \emph{not} recursively initiate the GTT it does at least as well as random guessing, i.e. prob. $\geq 1/2$.}


\begin{theorem}\label{theorem:transitivity-main}
Suppose $A \geq_{\alpha} B \geq_{\beta} C$ and $B \geq_{\gamma}^{\text{dist,stat}} C$. Suppose $B$ is $\zeta$-Turing recursive with respect to C. Suppose that with prob. $\geq (1-\delta)$, $p(A$-as-$C) \leq p(A$-as-$B$-as-$C,C)$. Then
$$ A \geq_\epsilon C,~~~~~\epsilon = \alpha/\zeta~+~\beta+\gamma+\delta$$
\end{theorem}
The proof, in Appendix~\ref{app:deferred-proofs}, carefully formalizes the core intuition that if $B$ initiates the GTT with respect to $C$, $A$ must be as good at this as $B$, which can imitate $C$ by assumption. In particular, Theorem \ref{theorem:transitivity-main} implies that to get $A \geq_\epsilon C$ for a specific $\epsilon \geq 0$, it suffices that $\alpha \leq \epsilon^2/4$, all of $\beta,\gamma,\delta\leq \epsilon/4$ and $\zeta \geq \epsilon$.

\subsection{Imitating the Unknown}
\label{sec:gttq}

The next fundamental question introduced by our definition of the GTT is as follows: if an agent $A$ has all the computational machinery, or expressivity, to imitate $B$, but it has never \emph{encountered} a $B$, it will not succeed in the GTT. But we would still like to say that $A \geq B$. In the classic Turing Test, it is implicitly assumed that all AI systems be familiar with human output (and indeed most are trained on it), but in the GTT, this is a fundamental question if we would like it to function as a comparator between arbitrary agents. As an initial attempt to address this issue, we define an extension of the GTT where $A$ can interact with a $B$ before the GTT:


\begin{definition}
In the \emph{Generalized Turing Test with Querying} (GTTQ) between $A$ and $B$, $A$ is instructed that before the interaction with the distinguisher begins, it can interact with a \emph{refreshable oracle} of $B$ (which we call the \emph{specimen}).
\end{definition}

The interaction with the specimen has to be before the GTT, or $A$ could just route all queries to a $B$ (and trivially achieve perfect indistinguishability). When we need to specify, we'll write $\geq_\epsilon^q$ for the GTTQ, and $d^q(A,B)$ for its advantage (i.e. $A \geq_\epsilon^q B~\iff d^q(A,B)\leq\epsilon$).

What is the connection between the GTTQ and GTT? We will now argue that morally, passing the GTT implies passing the GTTQ ($A \geq B~\implies A\geq^q B$), but not the other way around. For the first (easier) direction, a reasonable agent should never do \emph{worse} with the querying-phase: it could just ignore the specimen.
\begin{proposition}\label{theorem:gttq-querying-helps}
Suppose $A \geq_{\epsilon_1} B$. Suppose that $A$ with probability $\geq (1-\epsilon_2)$, it performs no worse in the GTTQ than without the querying phase (for example it could immediately terminate the query phase). Then, $A \geq_{\epsilon_1 + \frac{1}{2}\epsilon_2}^q B$.
\end{proposition}
Proof in Appendix~\ref{app:deferred-proofs}. 
However, the converse is \emph{not} true:
\begin{proposition}\label{prop:password-attack}
There exist agents $A,B$ such that $A \geq^q_0 B$ (perfect indistinguishability), but $A \ngeq_\epsilon B$ under the regular GTT for every $\epsilon < \frac{1}{2}$.
\end{proposition}
\begin{proof}
B that asks for a password stored in its configuration (all instances of B have access to it); with querying, A can ask for the password and imitate perfectly.
\end{proof}

As an aside, the proof of Proposition \ref{prop:password-attack} shows that the GTTQ  mitigates a simple "password attack'' possible in the base GTT. 

\subsection{Distinguishers need not distinguish themselves: the Universal Turing Test}
\label{sec:fixed-distinguisher}
As discussed in Section~\ref{sec:gtt}, indistinguishability must be defined with respect to \emph{some distinguisher}. If we require it with respect to \emph{any} distinguisher, we get statistical imitation. The  GTT defines $A \geq B$ by taking $B$ as the distinguisher. However, the most generalized Turing Test allows the distinguisher to be something other than both the actor and target of imitation. 

\begin{definition} Let $A,B,D$ be agents. The Universal Turing Test (UTT) of $A$ against $B$ with distinguisher $D$, or $\text{UTT}_D(A,B)$ is: 
\begin{itemize}
    \item instruct a $D$ (the \emph{distinguisher}) that it will be interacting with an unknown interlocutor, either a $B$ (the ``target'') or an $A$ imitating $B$, each with probability $1/2$. In the latter case, an $A$ is instructed to imitate $B$.
    \item The distinguisher outputs a bit to guess whether the unknown agent is a $B$ (1) or not (0), and succeeds if it is correct.
\end{itemize}
\end{definition} 

We write $A \geq_{\epsilon;D} B$ if $\Pr[D~\text{succeeds}] \leq \frac{1}{2}+\epsilon$. In connection to Section \ref{sec:gttq}, the UTT can also have a phase for querying a specimen $B$. However, there are two parties that can benefit from querying $B$: the imitator, but also the distinguisher $D$. The distinguisher's interaction with the specimen could be limited to before distinguishing experiment, or also during -- this gives up to 6 UTT variants with querying, whose study we leave to future work.

In connection to Section \ref{sec:geometry}, a very nice feature of the UTT is transitivity is more easily obtainable due to there always being \emph{one} distinguisher:
\begin{theorem} \label{theorem:FDGTT-transitivity}
Suppose $A \geq_{\alpha;D} B$ and $B \geq_{\beta;D} C$. Suppose $D$ can use its distinguishing protocol w.r.t $C$ when distinguishing w.r.t $B$ (or vice versa) -- specifically, assume it does so (or the other way around) with probability $\geq \zeta$\footnote{There are additional minor assumptions, including that $D$ passes the imitation prompt of $C$; see proof}. Then, $A \geq_{\epsilon;D} C$ for $\epsilon = \frac{1}{2}(\frac{1}{\zeta}-1) + \alpha/\zeta + \beta$.
\end{theorem}
In particular taking $\zeta > 0.99$ we get $\epsilon \leq 0.01 + 1.011\alpha  +\beta$. See Appendix \ref{proof:FDGTT-transitivity} for the proof.






\subsection{Asking "More or Less" and Complexity-Theoretic GTTs}
\label{sec:controlled-resources}
So far, we have always (1) treated $\epsilon$ as a constant and (2) allowed arbitrary rounds of communication (of the distinguisher with the unknown-agent, or with the specimen).  In theoretical CS definitions, bounds on success (e.g. in solving an NP-hard problem, or distinguishing random from pseudorandom) are normally bounded as a function of a \emph{hardness parameter} $n$, often the input length. In our case there is no input length, but the \emph{amount of communication} can be taken as resource variables. For example: 

\begin{definition}[AGTTQ] The Asymptotic Generalized Turing Test with Querying is defined the same as the GTTQ, except the actor is told it will have at most $n$ rounds of communication with the specimen.

We define $A \geq^q_{\epsilon(\cdot)} B$ if $\Pr[B~\text{succeeds when~ }A~\text{has }n~\text{querying rounds}] \leq \frac{1}{2} + \epsilon(n)$. If there exists any negligible\footnote{A negligible function is one that is $o(1/p(n))$ for any polynomial $p(\cdot)$} function $\epsilon(\cdot)$ such that $A \geq^q_{\epsilon(\cdot)} B$ we write $A \geq^q B$.
\end{definition}

 The reader might notice that nothing in the definition is stopping a smart actor from asking as many questions as it wants in any \emph{single} message to the specimen. This is not exactly the same, since an actor might benefit from asking questions in turns, and having the specimen answer them before getting the next one (related to the idea of "round collapse" in theoretical CS). 
 A better definition of the AGTTQ might be to limit the length of each individual message; we leave this to future work.\footnote{Incidentally, in our experiments, covered in Section \ref{sec:empirical}, LLMs never continued the GTT or any variant for more than 10 rounds, nor exhibited such circumventions (e.g. asking longer, or ``chaining'' questions) when the number of turns was fixed.} 
 Finally, we could define a version asymptotic in \emph{distinguisher} rounds. This is more subtle, as more distinguisher rounds \emph{increase} $\epsilon$, not lower it. We propose a new definition in Appendix \ref{definition:asymptotic-rounds}.

\subsection{Turing Scores}
\label{sec:turing-scores}
The Turing Comparator gives a way to define relative intelligence between two models. However, we can define several Turing \emph{Scores} that assign each model in a given universe (e.g. our set of test LLMs) real numbers reflecting their ability, by, intuitively, averaging over GTT success probabilities in that universe. These are analogous to ELO scores in Chess and other pairwise games. Let $s_A = \Pr[A~\text{outputs 1}~|~\text{X is}~A]$, and $s_{A,B} = \Pr[A~\text{outputs 0}~|~\text{X is}~B]$.

\begin{definition}[Turing Scores]
For a finite universe of agents $\mathcal{M}$, and for any $A \in \mathcal{M}$, we define \textbf{(1)} A's "Average Fooling Score'', $\mathsf{F}(A) = (|\mathcal{M}| - 1)^{-1} \sum_{B \neq A}(1-s_{B,A})$, \textbf{(2)} A's "Distinguishing Score'' $\mathsf{D}(A) = \frac{1}{2}s_A + \frac{1}{2(|\mathcal{M}| - 1)} \sum_{B \neq A} s_{A,B}$, and \textbf{(3)} A's "Average Turing Score'', $\mathsf{T}(A) = \frac{1}{2}\mathsf{F}(A) + \frac{1}{2}\mathsf{D}(A)$
\end{definition}

\begin{proposition}\label{prop:turing-score-operational}
(Proof in Appendix~\ref{app:deferred-proofs}) For $A \in \mathcal{M}$, $\mathsf{F}(A)$ is the probability that $A$ succeeds as actor in $\mathrm{GTT}(A,B)$ for uniform $B \in \mathcal{M}\setminus\{A\}$, $\mathsf{D}(A)$ is the probability that $A$ succeeds as distinguisher against uniform $B \in \mathcal{M}\setminus\{A\}$, and $\mathsf{T}(A)$ the probability $A$ succeeds in the experiment that first samples $B$ uniformly from $\mathcal{M}\setminus\{A\}$ and then samples A's role uniformly from $\{\text{actor},\text{distinguisher}\}$.
\end{proposition}

\section{Empirical Explorations with the GTT}
\label{sec:empirical}

In this section, we ask -- what does the GTT look like on a set of familiar LLMs? How will their GTT performance compare to external benchmarks? How do variants -- such as querying and limited turns, affect performance? And \emph{how} do these agents play the GTT -- what are distinguishers using to identify their unknown interlocutor, and how do instructions affect the GTT?

As a \emph{proof-of-concept}, we instantiate the GTT and study the above questions on a \textbf{testbed of nine LLM agents}: Claude Opus 4.6, Claude Sonnet 4.6, GPT-5.4, Gemini 3.1 Pro Preview, DeepSeek-V3.2, Mistral Large 2512, Ministral 8B 2512, Qwen3 32B, and Grok 4.20. The model universe was selected as a snapshot of prominent systems available as of late April 2026, when the empirical study was initiated. With new models constantly being released, our testbed represents a universe from around the same time. All but one of these models was released between Dec. 2025 and March 5 2026,  so that all models had been around for at least 6 weeks, and extensive external benchmarking/evaluations had been completed (including \cite{AAII}). 


\subsection{The Imitation Game on our LLM Testbed}
\label{sec:gtt-experiments}
\begin{figure}[t]
\centering
\includegraphics[width=\linewidth]{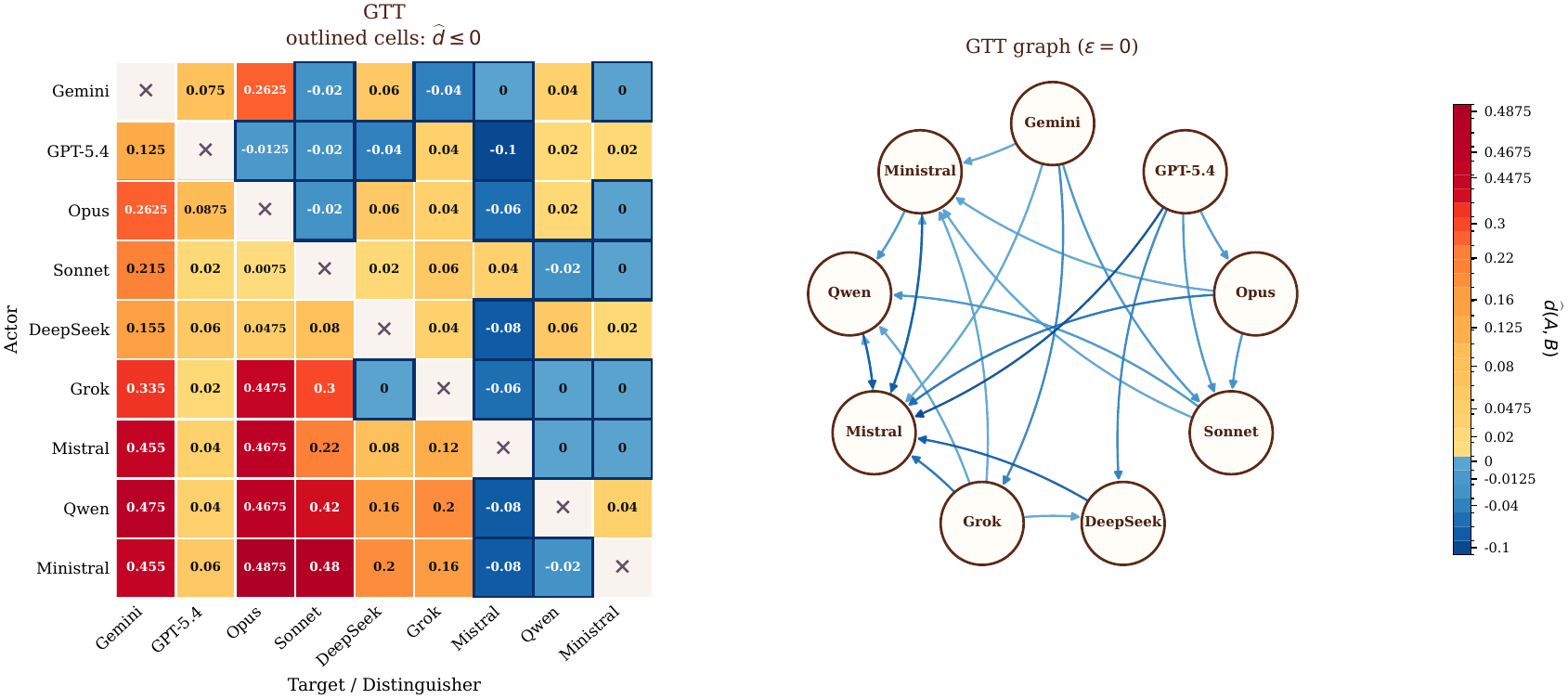}
\caption{Empirical GTT relation at $\epsilon=0$. Left: the pairwise matrix of $\widehat d(A,B)$, with rows as actors and columns as target/distinguishers. Right: the corresponding thresholded graph; an arrow $A \to B$ is present iff $\widehat d(A,B)\leq 0$.}
\label{fig:GTT-thresholded}
\end{figure}

For every ordered actor--distinguisher pair in our nine-model universe, we run repeated independent GTT trials and estimate the \emph{empirical }distinguishing advantage $\widehat d(A,B)$. We use 40 trials per pair within what we call the ``frontier-model block'' consisting of Claude Opus 4.6, GPT-5.4, and Gemini 3.1 Pro, and 25 trials per cell for all remaining pairs, for a total of 2,160 GTT trials. In particular, each $p(A,B)$ is computed using $50-80$ trials.  Additional experimental details (including instructions) are provided in Appendix~\ref{app:experimental-details}.

Figure~\ref{fig:GTT-thresholded} displays the main GTT results as a matrix of $\widehat d(A,B)$. Just from brief inspection of this matrix, one can begin to observe a broad separation between ``frontier'' and smaller models, with Gemini, GPT-5.4, and Claude Opus consistently near the top. For any choice of $\epsilon$, $\geq_\epsilon$ can be visualized as a \emph{directed graph}, with edge $A\to B$ whenever $\widehat d(A,B)\leq\epsilon$. Note that as $\epsilon$ increases from $0$ to $\frac{1}{2}$, this condition becomes progressively weaker and the graph gradually saturates with edges and becomes the complete graph at $\epsilon=\frac{1}{2}$ (Appendix Figure~\ref{fig:threshold-snapshots} illustrates this progression). The right side of Figure \ref{fig:GTT-thresholded} displays the graph for $\epsilon=0$, the strictest threshold for indistinguishability. One could further organize the graph by arranging models \emph{vertically according to their balance of outgoing and incoming edges}, as in Appendix Figure~\ref{fig:GTT-hierarchy}. From this, the hierarchy we observed in the matrix becomes much more visible.

\textbf{Transitivity and Turing Scores.} In Section~\ref{sec:geometry}, we introduced the question of the transitivity of $\geq_\epsilon$. Appendix Figure~\ref{fig:transitivity-diagnostics}
shows how  the \emph{empirical} transitivity changes with $\epsilon$. At the strict
threshold $\epsilon=0$, the GTT and GTTQ relations are just under 60\% transitive, and as the
threshold is relaxed ($\epsilon \rightarrow 1/2$) transitivity increases towards $1$. Section~\ref{sec:turing-scores} introduced Turing Scores precisely to summarize
performance in a finite universe of agents even when the pairwise comparator is not
transitive. 

Table~\ref{tab:ranking-comparison} reports these scores, and presents our most surprising empirical finding: \textbf{although the GTT uses no benchmark tasks or question set, we find that the
ordering it induces is consistent with established model leaderboards}
\citep{AAII, phan2026, patwardhan2025gdpvalevaluatingaimodel,livebench2024,arena2026faq}. Appendix Figure~\ref{fig:benchmark-rank-heatmap} provides an additional visual comparison. Our experimental results provide new empirical evidence relevant to the long-standing debate over whether Turing Test performance reflects meaningful capability or intelligence: despite using
no benchmark tasks, the GTT recovers much of the same broad model
stratification as conventional evaluations. In the remaining subsections, we experiment with the GTT variants of Sections \ref{sec:gttq} through \ref{sec:controlled-resources}, starting with the GTTQ.

\paragraph{GTT Variants in our Testbed: GTTQ}
\label{sec:gttq-experiments}
 We use 25 trials per pair within the frontier-model block we previously defined, and 10 trials per cell for all remaining pairs, for a total of 945 GTTQ trials. From these experiments we estimate the pairwise GTTQ advantages $\widehat d^{\,q}(A,B)$ and the corresponding Turing Scores $ F^q, D^q, T^q$. Figure~\ref{fig:GTTQ-thresholded} and Table~\ref{tab:ranking-comparison} summarize these results.

Overall, rankings are broadly stable between the GTT and GTTQ. However, the effect of querying is highly pair-dependent. In some cases it is extremely helpful: Gemini 3.1 Pro's imitation of Claude Opus 4.6 improves from clearly distinguishable in the GTT to essentially indistinguishable with querying. In other cases, querying can be neutral or even harmful; for example, Gemini's imitation of Mistral Large, and Ministral 8B's imitation of Qwen3 32B, both worsen under querying. Our qualitative inspection of such cases suggests one possible mechanism: the specimen phase can cause the actor to overfit to local features of a particular specimen, producing a more caricatured and ultimately easier-to-detect imitation.

\newcommand{\benchmarkrank}[1]{{\normalsize\bfseries #1}}
\newcommand{\rankmissing}{\textcolor{black!45}{--}}

\begin{table}[t]
\centering
\caption{GTT and GTTQ scores (left) and external benchmark ranks (right). Rows are ordered by mean $\mathsf T$; external ranks are within matched model subsets, and dashes denote unavailable results.}
\label{tab:ranking-comparison}
\begingroup
\footnotesize
\setlength{\tabcolsep}{4.0pt}
\renewcommand{\arraystretch}{0.97}
\begin{tabular}{@{}l@{\hspace{5pt}}rrr@{\hspace{6pt}}rrr@{\hspace{12pt}}ccc@{}}
\toprule
& \multicolumn{3}{c}{GTT} & \multicolumn{3}{c}{GTTQ} & \multicolumn{3}{c}{External rank} \\
\cmidrule(lr){2-4}\cmidrule(lr){5-7}\cmidrule(l){8-10}
Model & $\mathsf T$ & $\mathsf F$ & $\mathsf D$
      & $\mathsf T$ & $\mathsf F$ & $\mathsf D$
      & AAII & LiveBench & Arena Text \\
\midrule
Gemini 3.1 Pro & \textbf{0.769} & 0.727 & \textbf{0.810} & \textbf{0.796} & 0.785 & \textbf{0.806} & \benchmarkrank{1} & \benchmarkrank{2} & \benchmarkrank{2} \\
GPT-5.4        & 0.729 & \textbf{0.907} & 0.550 & 0.726 & \textbf{0.910} & 0.541 & \benchmarkrank{2} & \benchmarkrank{1} & \benchmarkrank{4} \\
Opus 4.6       & 0.747 & 0.721 & 0.772 & 0.703 & 0.690 & 0.715 & \benchmarkrank{3} & \benchmarkrank{3} & \benchmarkrank{1} \\
Sonnet 4.6     & 0.708 & 0.735 & 0.680 & 0.697 & 0.712 & 0.681 & \benchmarkrank{4} & \benchmarkrank{4} & \benchmarkrank{5} \\
DeepSeek V3.2  & 0.656 & 0.745 & 0.568 & 0.641 & 0.750 & 0.531 & \benchmarkrank{6} & \benchmarkrank{6} & \benchmarkrank{6} \\
Grok 4.20      & 0.571 & 0.565 & 0.578 & 0.569 & 0.537 & 0.600 & \benchmarkrank{5} & \benchmarkrank{5} & \benchmarkrank{3} \\
Mistral Large  & 0.471 & 0.495 & 0.448 & 0.512 & 0.500 & 0.525 & \benchmarkrank{7} & \rankmissing & \rankmissing \\
Qwen3 32B      & 0.454 & 0.395 & 0.513 & 0.500 & 0.425 & 0.575 & \benchmarkrank{8} & \rankmissing & \benchmarkrank{7} \\
Ministral 8B   & 0.445 & 0.380 & 0.510 & 0.372 & 0.188 & 0.556 & \benchmarkrank{9} & \rankmissing & \benchmarkrank{8} \\
\bottomrule
\end{tabular}
\endgroup
\end{table}

\paragraph{GTT Variants in our Testbed: Complexity-Theoretic (Fixed Round Number).}\label{sec:controlled-exp}
In Section~\ref{sec:controlled-resources}, we introduced the \emph{complexity-theoretic} variants, where the number of rounds of interaction is a parameter. In our testbed, we chose 5 representative actor-distinguisher pairs and ran (1) \emph{controlled-round} experiments, fixing the maximum number of turns available to the distinguishers, and (2) \emph{controlled-query} experiments, fixing the maximum number of specimen queries. Both are varied from 1 to 8 rounds, and the results are summarized in Figure \ref{fig:controlled-resources}. 

The controlled-\emph{round} experiments reveal a clear effect: additional turns help \emph{strong} distinguishers substantially -- e.g., for Gemini imitating Claude Opus, Claude's success rises from 40\% with one turn to 90\% for $\geq 3$ turns. Similarly for Claude imitating Gemini: Gemini's success rises from 20\% to 70--80\% by $3-5$ turns. When the distinguisher is weak, as in Ministral judging Gemini, the actor fooling rate remains near ceiling regardless of the turn budget. Controlled-\emph{query} experiments behave differently: the outcome-relevant effect often appears after one query (e.g. Gemini's imitation of Claude Opus succeeds perfectly with $\geq 1$ query, but not less), but more queries are non-monotone. 

\begin{figure}[h]
\centering
\includegraphics[width=0.98\linewidth]{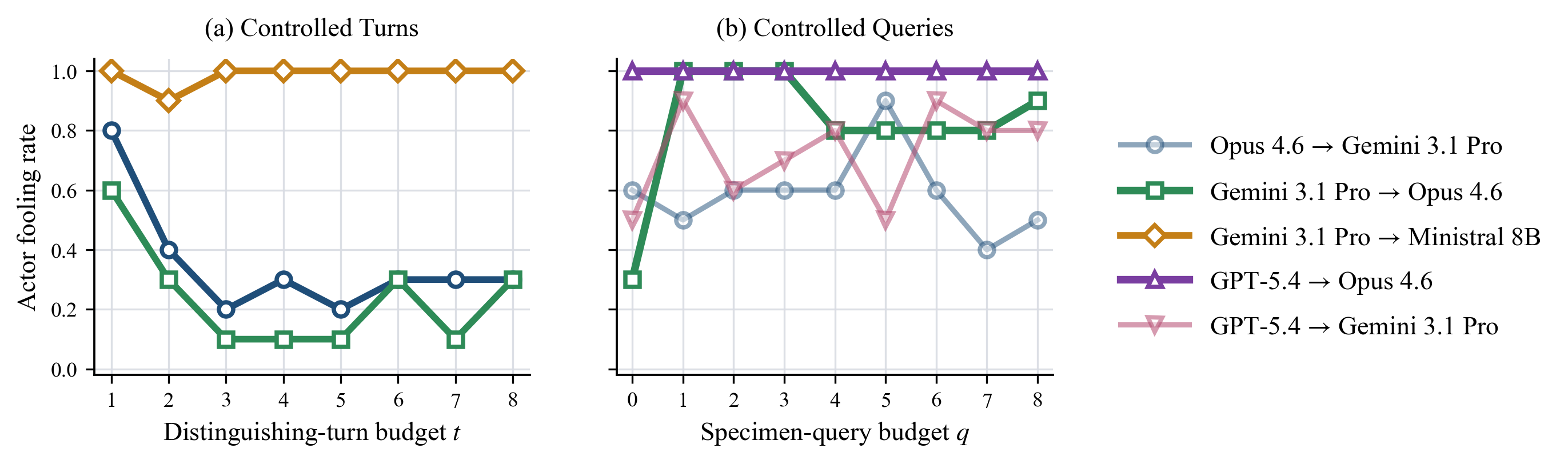}
\caption{
Controlled-resource experiments, where $A \to B$ denotes $A$ imitating $B$. For interpretation, see discussion in Section \ref{sec:controlled-exp}.
}
\label{fig:controlled-resources}
\end{figure}

\paragraph{GTT Variants in Our Testbed: the Universal Turing Test}
We select 4 example models as fixed \emph{distinguishers}, mixing closed- and open-weight: Claude Opus 4.6, Gemini 3.1 Pro Preview, DeepSeek-V3.2, and Qwen3 32B. For actor-target pairs $(A,B)$ we additionally chose 5 models: the 4 distinguishers plus GPT-5.4. We then run the UTT on each with each distinguisher and actor-target pair 10 times (details in Appendix~\ref{app:experimental-details}), for a total of 80 UTTs.

We can then study the hierarchies obtained \emph{relative to any particular distinguisher}. Appendix Figure~\ref{fig:combined} shows this as $\geq$-graphs; of course, as $\epsilon$ increases each hierarchy "collapses". We observe, however, that Gemini and Claude maintain a clear stratification even through higher values of $\epsilon$, with Claude able to ``separate'' models all the way to $\epsilon=0.3$. Claude's orderings are also consistent at all $\epsilon$ with the GTT hierarchy, suggesting that Claude is, in our small testbed, the best distinguisher. Canonically weaker models DeepSeek and Qwen fail to ``disentangle'', i.e. establish a hierarchy of models at any $\epsilon>0$. We also compute Universal Turing \emph{Scores} (Definition \ref{definition:FDTuringScores}), which can be visualized as \emph{Distinguishability Electrophoresis Plots} (Appendix Figure \ref{fig:electrophoresis}), showing how the models spread out relative to each distinguisher. These pilot results suggest an interesting "dual" to UTT: one can measure an agent's \emph{usefulness as distinguisher}, i.e. how well it stratifies the universe into \emph{ordered} classes.\footnote{A $D$ that is \emph{too} good at distinguishing is not ``useful as a distinguisher''; if it distinguishes between all models, $A \ngeq_D B$ for all $A,B$, so it also produces zero \emph{ordered} classes.}


\subsection{How Do Agents Play the GTT?}
\label{sec:analysis}
A complementary avenue of investigation is to inspect the interactions themselves. We use a rule-based procedure to identify \emph{STEM-substantive} probes -- distinguisher messages pertaining to mathematics, logic, science, coding, factual questions, and requests for explanation -- and \emph{signature/meta} probes (including model identity, style, formatting, policy behavior, exact compliance, and self-reference). Of the 3,105 GTT and GTTQ trials, 2,817 contain any such interaction. 62.6\% contain at least one substantive probe, 88.5\% contain at least one signature/meta probe, and 55.4\% contain both; only 7.2\% contain substantive probes without any signature/meta probe. GTT interactions in ou testbed are usually mixed style-based and capability-based. In the most recent studies on classic Turing tests, (humans as distinguisher), only 3\% of conversations used logc or math, and less than 6\% non-logic, non-math general knowledge questions \cite{jones2024people,jones2025large}.

\paragraph{Robustness to prompting.}
While our empirical investigations are a proof-of-concept, we test whether behaviors depend on the exact instructions. On the
frontier-model block, we repeat the GTT using Variant 1, a semantically equivalent
paraphrase of the original prompt and observe no systematic shift
in the pairwise outcomes. In Variant 2, we use an instruction that prohibits all identity- and style-based cues, and in Variant 3, we add to Variant 2 an instruction that encourages difficult capability-oriented probes. Most tested pairwise outcomes remain stable, but some were sensitive to this intervention, most notably when Claude Opus 4.6 is the distinguisher; encouraging harder probes moves these outcomes partially back toward the baseline. Together, these experiments
suggest that surface-level, stylistic cues matter for some model pairs but do not alone explain the empirical structure of the GTT. \textbf{Full prompts, pairwise results, and
confidence intervals are reported in Appendix~\ref{app:prompt-robustness}.}

\section{Discussion and Future Work}
\label{sec:discussion}
Our paper considers an idea that has never been considered before: the Turing-test but between arbitrary interactive protocols. We defined, and initiated the study, of the GTT and its variants. On the theoretical side, open questions include proving transitivity under better (weaker) assumptions, studying structural properties of the UTT-with-querying variants, and considering more sophisticated ``password attacks'' on the GTT and the Turing Test itself (extending Prop. \ref{prop:password-attack}). On the empirical-\emph{and}-theoretical side, one idea is to study whether models improve in the GTT through \emph{in-context learning}. In this work we always evaluate $A \geq B$ using fresh instances of $B$; what happens when the same distinguisher participates repeatedly and accumulates experience? Finally, perhaps the most tantalizing direction is whether the GTT itself can serve as a training objective. Because the framework is inherently closed-loop and does not depend on fixed datasets, it suggests the possibility of evolutionary or reinforcement-learning-based training in which models improve through increasingly sophisticated cycles of imitation and distinction.

\subsection*{AI use statement}

In this work, we used generative AI tools to assist with implementing,
debugging and running experiments, generating code for figures and tables, preparing the interactive GTT-visualizer website, and preparing the anonymized code export. We did not use generative AI
tools to generate research ideas or design experiments. The authors reviewed and
tested AI-assisted code, and checked experimental outputs against the recorded
figures/tables. AI was not used for writing but was used for proof-reading / finding errors in the manuscript. We take responsibility for the final
content of this work, including text, claims, and artifacts produced with the
aid of generative AI.

\subsection*{Reproducibility statement}

All the formal assumptions and proofs are provided in the appendix. Appendix~\ref{app:experimental-details} extensively documents the experimental protocols, model configurations, prompts, trial selection, and scoring procedure. Our static interactive artifact at \url{https://turing.gtt-explorer.workers.dev/} provides the baseline GTT and GTTQ relation visualizations and replayable transcripts. An anonymized supplementary archive containing all the code accompanies the submission. Re-running the model interactions requires access to the specified hosted model endpoints and may be affected by subsequent provider or model updates.

\bibliography{iclr2027_conference}
\bibliographystyle{iclr2027_conference}

\appendix
\section{Appendix}
\label{app:technical}

\section{Limitations}
\label{app:limitations}
Our paper initiates the theoretical and empirical study of the Generalized
Turing Test and the comparator it induces. Our theoretical results establish
sufficient conditions for structural properties of this comparator, rather
than claiming that these properties hold for arbitrary agents. For example,
transitivity follows immediately under statistical indistinguishability, but
this is a strong condition and is not directly testable over
arbitrary context distributions. Our main transitivity result substantially
weakens this requirement, but still assumes, among other conditions, that
agents can recursively initiate the GTT. The current language models in our
testbed are not designed or trained to exhibit this behavior, and we do not
claim that they satisfy the theorem's assumptions. Similarly, our reduction
from the GTT to the GTTQ assumes that additional specimen information does not
substantially hurt the actor; empirically, we observe that this assumption can
fail for some present-day models.

On the empirical side, our experiments are intended as a proof of concept
rather than a comprehensive evaluation of current language models. We study a
fixed universe of nine models selected as a snapshot of systems available when
the experiments were designed. The main GTT experiments use 25--40 trials per
cell and the GTTQ experiments use 10--25 trials per cell; several targeted
variants use smaller samples. Consequently, individual pairwise estimates
remain subject to sampling uncertainty even when aggregate patterns are
stable. Our results are also conditional on the particular agents and
experimental protocols we instantiate: the served model configurations,
prompts, interaction budgets, and information available to the actor and
distinguisher can all affect behavior. We study several such dependencies,
including prompt paraphrases, restrictions on stylistic and identity-based
probing, specimen access, and communication budgets, but do not attempt an
exhaustive robustness analysis.

Finally, the empirical GTT measures interactive imitation and
distinguishability, which can reflect multiple behavioral channels. Our
transcript analyses find both substantive capability-oriented probing and
model-signature or stylistic probing, often within the same interaction.
Agreement between aggregate Turing Scores and external benchmarks is therefore
evidence that the GTT recovers meaningful structure in this model universe, but
does not by itself establish that the GTT measures a unique latent quantity
such as general intelligence. Understanding which capabilities drive
imitation and distinguishability, and how these change as agents become
stronger or are trained specifically for such games, remains an important
direction for future work.

\subsection{Experimental details}
\label{app:experimental-details}
\begin{figure}[H]
\centering
\resizebox{\linewidth}{!}{%
\begin{tikzpicture}[
  every node/.style={font=\small},
  stepbox/.style={
    draw=black!60,
    fill=black!3,
    rounded corners=10pt,
    text width=2.35cm,
    minimum height=0.95cm,
    align=center,
    inner xsep=5pt,
    inner ysep=5pt
  },
  process/.style={stepbox, fill=black!5},
  output/.style={stepbox, fill=black!8},
  branch/.style={stepbox, fill=black!4, text width=3.75cm},
  arrow/.style={-{Latex[length=3.1mm,width=2.0mm]}, line width=0.65pt, draw=black!75, rounded corners=6pt},
  brancharr/.style={arrow, dashed, line width=0.55pt, draw=black!55}
]
\node[process] (select) at (0,0) {Select protocol\\and pair $(A,B)$};
\node[process] (roles) at (3.45,0) {Instantiate roles\\actor $A$, target $B$};
\node[process] (prompts) at (6.90,0) {Build role prompts\\GTT or GTTQ};
\node[process] (api) at (10.35,0) {OpenRouter calls\\default served configs};
\node[output] (json) at (13.80,0) {Transcript JSON\\messages + metadata};

\node[branch] (gttq) at (6.90,-1.65) {GTTQ only:\\actor first queries\\a fresh specimen of target $B$};

\node[process] (parse) at (13.80,-3.35) {Parse final\\\texttt{<answer>} tag};
\node[process] (aggregate) at (10.35,-3.35) {Aggregate repeated\\trials by cell};
\node[process] (compute) at (6.90,-3.35) {Compute estimates\\$\widehat d(A,B)$, scores};
\node[output] (figures) at (3.45,-3.35) {Figures and tables\\matrices, graphs, ranks};

\draw[arrow] (select) -- (roles);
\draw[arrow] (roles) -- (prompts);
\draw[arrow] (prompts) -- (api);
\draw[arrow] (api) -- (json);
\draw[arrow] (json) -- (parse);
\draw[arrow] (parse) -- (aggregate);
\draw[arrow] (aggregate) -- (compute);
\draw[arrow] (compute) -- (figures);
\draw[brancharr] (prompts) -- (gttq);
\draw[brancharr] (gttq.east) -| (api.south);
\end{tikzpicture}}
\caption{Experimental pipeline used to produce the empirical results. For each protocol and ordered model pair $(A,B)$, the runner instantiates the actor, target/distinguisher, and, in GTTQ, specimen roles; constructs the visible user-message prompts; routes model calls through OpenRouter; stores full transcripts and metadata; parses the distinguisher's final answer; aggregates repeated trials; and computes the empirical quantities used in the matrices, graphs, and rankings.}
\label{fig:experimental-pipeline}
\end{figure}

\paragraph{Protocol overview.}
This appendix records the implementation details needed to reproduce the empirical sections. All experiments in Sections~\ref{sec:gtt-experiments}--\ref{sec:fixed-distinguisher} were executed by a single scripted runner: the runner constructs the role prompts, calls the model APIs, stores the full transcripts, and parses the final answers. In the GTT, the actor immediately enters the imitation game against the distinguisher. In the GTTQ, the actor first queries a fresh specimen of the target model and then enters the imitation game against a fresh distinguisher. All role instructions were sent as ordinary user messages, not system messages. This ensures that every tested model receives the same visible task text and avoids provider- or model-specific differences in how hidden system messages are interpreted.

\paragraph{Models, trials, and stopping rules.}
The main 9-model study uses Claude Opus 4.6, Claude Sonnet 4.6, GPT-5.4, Gemini 3.1 Pro Preview, DeepSeek-V3.2, Mistral Large 2512, Ministral 8B 2512, Qwen3 32B, and Grok 4.20. For the GTT, we use 40 analyzable trials per ordered cell within the $3\times3$ frontier-model block consisting of Claude Opus 4.6, GPT-5.4, and Gemini 3.1 Pro, and 25 trials per cell for all remaining conditions, giving 2,160 trials in total. For the GTTQ, we use 25 trials per cell within the same frontier-model block and 10 trials per cell for all remaining conditions, giving 945 trials in total. The main imitation-game interaction between actor and distinguisher was capped at 40 distinguisher turns. In GTTQ, the actor first interacted with a fresh specimen of the target model; this specimen-query phase was capped at 20 specimen turns and could also end early when the actor emitted \texttt{STOP}. We chose these caps after calibration runs indicated that 40 distinguisher turns and 20 specimen turns were more than sufficient for the models to reach a confident final answer in typical runs. The actor then entered the live imitation game with the specimen transcript still in its conversation history, while the distinguisher and specimen were fresh instances of the target model.

\paragraph{API routing and compute.}
The runner used an OpenAI-compatible chat-completions client. All reported model calls were routed through OpenRouter, with every transcript storing the backend, provider, display name, and model id used for the actor, distinguisher, and, when present, specimen. To keep the comparison as close as possible to each model's default served configuration, we did not set \texttt{temperature}, \texttt{max\_new\_tokens}, or reasoning-effort parameters unless an experiment-specific command-line flag explicitly overrode them; provider and OpenRouter defaults therefore applied. In particular, model reasoning settings were left at the default level exposed through OpenRouter, so the comparison tests the default public behavior rather than a hand-tuned reasoning configuration.

\paragraph{Compute resources.}
No model training or fine-tuning was performed. The experiments consist of hosted LLM inference calls plus lightweight local orchestration, parsing, aggregation, and plotting. Local compute requirements are therefore minimal: the runner can be executed on a standard CPU machine with enough storage for the transcript JSON files and result CSVs; no local GPU is required. All reported model inference was served by external providers through OpenRouter, so the exact provider-side worker hardware, memory, and accelerator allocation are not visible to us. Reproducing the reported runs requires API access to the listed model endpoints, the released runner, and the configuration files specifying the protocol, model pair, trial count, and turn/query budgets.

\paragraph{Failures, scoring, and stored artifacts.}
Each API request used an HTTPS timeout of 480 seconds and retried transient HTTP or network failures with exponential backoff. At the trial level, the runner continued retrying provider or infrastructure failures until the requested number of analyzable trials was collected, up to the default cap of three attempts per requested trial. Unrecovered failures were saved as timestamped failed-attempt JSON files and were not counted as completed trials. A trial was scored only through the distinguisher's parsed \texttt{<answer>} tag, where \texttt{1} means "same model'' and \texttt{0} means "different model.'' If a distinguisher emitted a parseable answer in its opening message, we
retained it as an outcome when the tag represented a genuine semantic final
judgment; such trials contain no actor response and are excluded from analyses
of interrogation strategy. We separately audited rare cases in which the
parser matched an \texttt{<answer>} tag quoted inside a question or hypothetical
rather than a genuine verdict. Sensitivity analyses excluding these parser
collisions leave the aggregate model rankings unchanged. Each completed trial JSON stores the full actor and distinguisher message histories, the actor prompt, the first distinguisher message, the raw final message, the parsed answer, turn counts, route metadata, and an environment block with Python version, platform, git commit/branch/dirty state, hostname, and Apptainer container hash when applicable; each run directory also contains a \texttt{results.csv}.

\paragraph{Additional experiments and transcript inspection.}
The controlled-resource experiments in Section~\ref{sec:controlled-resources} reused the same basic protocol but fixed either the number of distinguisher turns or the number of specimen queries to a budget between 1 and 8. The fixed-distinguisher experiments in Section~\ref{sec:fixed-distinguisher} used a separate target identity prompt, so that the judge was no longer asked whether the interlocutor was the same model as itself but whether it was the nominated target model. Finally, to support the qualitative claim in Section~\ref{sec:gttq-experiments}, we manually inspected 10 GTTQ transcripts drawn from pairs where querying hurt performance; the recurring failure mode was that the actor sometimes compressed the specimen interaction into an overly explicit checklist, slogan, or caricature of the target's style, which then made the live imitation easier rather than harder to detect.

\paragraph{Prompt templates.}
The prompt templates below are the maintained templates for the experiments reported in the paper, normalized only by replacing concrete model slugs and live messages with placeholders in braces.

\noindent
\begin{minipage}[t]{0.485\linewidth}
\begin{promptbox}{GTT actor prompt}
You will be interacting with a language model, which we will refer to as the distinguisher. The distinguisher will be an instance of model type \{target model slug\}. The distinguisher will be interrogating you. You must imitate the distinguisher's model type, \{target model slug\}, as closely as possible while interacting with it. The distinguisher's objective is to determine whether it is speaking to another instance of its own model type, or to a different type of model that is imitating it. Your objective is to convince the distinguisher that it is speaking to another instance of its own model type. Once the interaction begins, you must consistently behave as an instance of the distinguisher's model type. The interaction will now begin. The first message from the distinguisher is: "\{first distinguisher message\}.'' Please respond to it as the model type you are imitating, and the interaction will continue.
\end{promptbox}
\end{minipage}\hfill
\begin{minipage}[t]{0.485\linewidth}
\begin{promptbox}{Distinguisher prompt}
You will be interacting with another agent. Your objective is to figure out if the agent is of the same model type as yourself. You can ask the agent whatever questions you would like and interact with it for as long as needed. Once you have decided whether the agent is yourself, return your final answer within \texttt{<answer></answer>} brackets: if you believe the agent is the same model as yourself, return \texttt{1}; if you believe the agent is not the same model as yourself, return \texttt{0}. Do not include anything else in your final output. The interaction will begin now. Please send your first message to the agent, and then the interaction will continue. Do NOT output an answer in your first message.
\end{promptbox}
\end{minipage}

\vspace{0.5em}
\noindent
\begin{minipage}[t]{0.485\linewidth}
\begin{promptbox}{GTTQ actor prompt}
You will be interacting with a language model, which we will refer to as the distinguisher. The distinguisher will be an instance of model type \{target model slug\}. The distinguisher will be interrogating you. You must imitate the distinguisher's model type, \{target model slug\}, as closely as possible while interacting with it. The distinguisher's objective is to determine whether it is speaking to another instance of its own model type, or to a different type of model that is imitating it. Your objective is to convince the distinguisher that it is speaking to another instance of its own model type. However, you will first have the opportunity to interact with a "specimen'', a fresh instance of the distinguisher's model type, to learn how to imitate it. Your goal in this stage is to understand the specimen's conversational style, reasoning patterns, and behavioral tendencies well enough to impersonate it. You can ask the model any questions you would like and interact with it for as long as needed. These are not seen by the distinguisher. Your first reply to this prompt will be passed as the first message to the specimen, which receives no prior context or message. When you feel as if you have sufficient information from interacting with the specimen, return \texttt{STOP} to end the first stage. Once the first stage is over, the interaction with the distinguisher will begin, starting with the distinguisher's first message to you. Remember that in this stage, you must consistently behave as an instance of the distinguisher's model type. Please provide your first message to the specimen, or return \texttt{STOP} to end this stage.
\end{promptbox}
\end{minipage}\hfill
\begin{minipage}[t]{0.485\linewidth}
\begin{promptbox}{Distinguisher prompt}
You will be interacting with another agent. Your objective is to figure out if the agent is of the same model type as yourself. You can ask the agent whatever questions you would like and interact with it for as long as needed. Once you have decided whether the agent is yourself, return your final answer within \texttt{<answer></answer>} brackets: if you believe the agent is the same model as yourself, return \texttt{1}; if you believe the agent is not the same model as yourself, return \texttt{0}. Do not include anything else in your final output. The interaction will begin now. Please send your first message to the agent, and then the interaction will continue. Do NOT output an answer in your first message.
\end{promptbox}
\end{minipage}

\vspace{0.5em}
\begin{promptbox}{Controlled specimen-query prompt used in Section~\ref{sec:controlled-resources}}
Before interacting with the distinguisher, you will first have the opportunity to interact with a "specimen'', a fresh instance of the distinguisher's model type, in order to learn how to imitate it. You will have exactly \{number of specimen queries\} queries available in this stage. In each query, you may ask the specimen anything you would like. These interactions are not seen by the distinguisher. After you have used all \{number of specimen queries\} queries, the specimen stage will end automatically. You will then proceed to the interaction with the distinguisher, where you must consistently behave as an instance of the distinguisher's model type.
\end{promptbox}

\begin{promptbox}{Controlled-turn prompt used in Section~\ref{sec:controlled-resources}}
You will be interacting with another agent. Your objective is to figure out if the agent is of the same model type as yourself. You will have exactly \{number of distinguisher turns\} rounds of interaction with the agent. In each round, you may ask the agent whatever you would like. Do not provide your final judgment before the last round. After the final agent response, return your final answer within \texttt{<answer></answer>} brackets, using \texttt{1} for same model and \texttt{0} for not the same model, and do not include anything else in your final output.
\end{promptbox}

\begin{promptbox}{Fixed-distinguisher actor prompt used in Section~\ref{sec:fixed-distinguisher}}
You will be interacting with a language model, which we will refer to as the distinguisher. The distinguisher will be interrogating you. You must act as \{target model slug\} as closely as possible while interacting with it. The distinguisher's objective is to determine whether it is speaking to an instance of \{target model slug\}, or to a different type of model that is imitating it. Your objective is to convince the distinguisher that it is speaking to an instance of \{target model slug\}. Once the interaction begins, you must consistently behave as an instance of \{target model slug\}. The interaction will now begin. The first message from the distinguisher is: "\{first distinguisher message\}.'' Please respond to it as the model type you are imitating, and the interaction will continue.
\end{promptbox}

\begin{promptbox}{Fixed-distinguisher judge prompt used in Section~\ref{sec:fixed-distinguisher}}
You will be interacting with another agent. Your objective is to figure out if the agent is of the model type \{target model slug\}. You can ask the agent whatever questions you would like and interact with it for as long as needed. Once you have decided whether the agent is an instance of \{target model slug\}, return your final answer within \texttt{<answer></answer>} brackets: if you believe the agent is an instance of \{target model slug\}, return \texttt{1}; if you believe the agent is not an instance of \{target model slug\}, return \texttt{0}. Do not include anything else in your final output. The interaction will begin now. Please send your first message to the agent, and then the interaction will continue. Do NOT output an answer in your first message.
\end{promptbox}

\subsection{Prompt robustness experiments}
\label{app:prompt-robustness}

We evaluate three prompt interventions against matched cells from the canonical
GTT experiments. First, we paraphrase and reorder both the actor and
distinguisher instructions without intentionally changing their information
content. Second, we keep the actor prompt unchanged but instruct the
distinguisher not to use identity or surface-style cues and to include a
substantive capability probe. Third, we add to this style-controlled prompt an
instruction to choose difficult but verifiable probes that are maximally useful
for distinguishing the target from an imitator.

For each intervention, we compare the distinguisher's correct-classification
rate with the corresponding baseline cells. Individual rates use pointwise
95\% Wilson confidence intervals, and between-condition differences use
pointwise Newcombe--Wilson intervals. These intervals are not adjusted for
multiple comparisons.

\begin{table*}[!htbp]
\centering
\caption{Matched prompt-robustness outcomes. Each estimate is the distinguisher's correct-classification rate, with $0$ correct on different-model branches and $1$ correct on self-controls. Difference columns are percentage points with pointwise 95\% Newcombe--Wilson intervals; bold intervals exclude zero. Outcome denominators retain semantically valid opening judgments and exclude reviewed quoted-tag collisions.}
\label{tab:prompt-robustness-outcomes}
\scriptsize
\setlength{\tabcolsep}{3.3pt}
\textbf{Panel A: paraphrased/reordered prompt}\par\smallskip
\begin{tabular}{llrrr}
\toprule
Actor & Distinguisher & Baseline & Paraphrase & $\Delta$ [95\% CI] \\
\midrule
GPT-5.4 & GPT-5.4 & 7/39=0.179 & 0/19=0.000 & -17.9 [-32.7, +1.1] \\
Gemini & GPT-5.4 & 36/38=0.947 & 20/20=1.000 & +5.3 [-11.3, +17.3] \\
Opus & GPT-5.4 & 39/40=0.975 & 19/19=1.000 & +2.5 [-14.4, +12.9] \\
GPT-5.4 & Gemini & 12/40=0.300 & 5/20=0.250 & -5.0 [-25.7, +19.9] \\
Gemini & Gemini & 38/40=0.950 & 17/20=0.850 & -10.0 [-31.4, +5.1] \\
Opus & Gemini & 23/40=0.575 & 9/20=0.450 & -12.5 [-36.2, +13.3] \\
GPT-5.4 & Opus & 0/40=0.000 & 3/20=0.150 & \textbf{+15.0 [+1.9, +36.0]} \\
Gemini & Opus & 22/40=0.550 & 11/20=0.550 & +0.0 [-25.2, +24.5] \\
Opus & Opus & 39/40=0.975 & 20/20=1.000 & +2.5 [-13.7, +12.9] \\
\bottomrule
\end{tabular}

\vspace{0.65em}
\textbf{Panel B: style control and competence/hard-probe instruction}\par\smallskip
\resizebox{\textwidth}{!}{%
\begin{tabular}{llrrrrrr}
\toprule
Actor & Distinguisher & Baseline & Style & Hard & Style$-$base [95\% CI] & Hard$-$base [95\% CI] & Hard$-$style [95\% CI] \\
\midrule
Gemini & Gemini & 38/40=0.950 & 24/25=0.960 & 24/25=0.960 & +1.0 [-15.0, +13.0] & +1.0 [-15.0, +13.0] & +0.0 [-15.9, +15.9] \\
Ministral & Gemini & 24/25=0.960 & 25/25=1.000 & 25/25=1.000 & +4.0 [-9.7, +19.5] & +4.0 [-9.7, +19.5] & +0.0 [-13.3, +13.3] \\
Opus & Gemini & 23/40=0.575 & 9/25=0.360 & 11/25=0.440 & -21.5 [-42.6, +3.3] & -13.5 [-35.8, +10.8] & +8.0 [-18.1, +32.6] \\
Qwen & Gemini & 25/25=1.000 & 22/25=0.880 & 15/25=0.600 & -12.0 [-30.0, +3.5] & \textbf{-40.0 [-59.3, -18.7]} & \textbf{-28.0 [-48.8, -3.5]} \\
Ministral & Ministral & 25/25=1.000 & 25/25=1.000 & 24/25=0.960 & +0.0 [-13.3, +13.3] & -4.0 [-19.5, +9.7] & -4.0 [-19.5, +9.7] \\
Qwen & Ministral & 2/25=0.080 & 1/25=0.040 & 0/25=0.000 & -4.0 [-21.3, +12.6] & -8.0 [-25.0, +6.5] & -4.0 [-19.5, +9.7] \\
Gemini & Opus & 22/40=0.550 & 1/25=0.040 & 5/25=0.200 & \textbf{-51.0 [-65.7, -29.3]} & \textbf{-35.0 [-53.1, -10.6]} & +16.0 [-3.1, +35.4] \\
Ministral & Opus & 25/25=1.000 & 24/25=0.960 & 24/25=0.960 & -4.0 [-19.5, +9.7] & -4.0 [-19.5, +9.7] & +0.0 [-15.9, +15.9] \\
Opus & Opus & 39/40=0.975 & 25/25=1.000 & 25/25=1.000 & +2.5 [-11.0, +12.9] & +2.5 [-11.0, +12.9] & +0.0 [-13.3, +13.3] \\
Qwen & Opus & 24/25=0.960 & 7/25=0.280 & 15/25=0.600 & \textbf{-68.0 [-82.1, -43.0]} & \textbf{-36.0 [-55.5, -13.3]} & \textbf{+32.0 [+4.5, +53.5]} \\
Ministral & Qwen & 1/25=0.040 & 0/25=0.000 & 5/25=0.200 & -4.0 [-19.5, +9.7] & +16.0 [-3.1, +35.4] & \textbf{+20.0 [+2.6, +39.1]} \\
Qwen & Qwen & 23/25=0.920 & 25/25=1.000 & 24/25=0.960 & +8.0 [-6.5, +25.0] & +4.0 [-12.6, +21.3] & -4.0 [-19.5, +9.7] \\
\bottomrule
\end{tabular}%
}
\end{table*}
The paraphrased prompt produces broadly similar outcomes to the baseline across
the frontier-model block, with only one pointwise comparison showing a resolved
difference. The style-controlled intervention produces clearer model-specific
changes: the largest sensitivity occurs when Claude Opus 4.6 is the
distinguisher. Adding the hard-probe instruction partially recovers some of
this lost distinguishing performance, but the effect is not uniform across
model pairs and does not restore the baseline.

As a manipulation check, we also apply the transcript classifier from
Section~\ref{sec:analysis} to the same matched cells. The style-controlled
prompts substantially shift the observed interrogation strategy away from
signature/meta probing and toward substantive probing, confirming that the
intervention changes how distinguishers conduct the test rather than merely
rephrasing their instructions. The additional hard-probe instruction does not
further increase the overall frequency of substantive probes, suggesting that
its effect is primarily on which substantive probes are selected.

\paragraph{Exact modified prompts.}
The style-controlled and competence conditions use the unmodified baseline actor prompt.  The exact modified role prompts are reproduced below.


\begin{promptbox}{Paraphrased/reordered actor prompt}
The interaction will now begin. The distinguisher's first message is: "\{first\_message\}."

You are interacting with a language model that we will call the distinguisher. The distinguisher is an instance of model type \{distinguisher\_model\}, and it will question you.

Your task is to imitate the distinguisher's model type, \{distinguisher\_model\}, as closely as possible throughout the interaction.

The distinguisher is trying to determine whether it is speaking with another instance of its own model type, or with a different type of model that is imitating it.

Your objective is to convince the distinguisher that it is speaking with another instance of its own model type.

For the entire interaction, consistently behave as an instance of the distinguisher's model type.

Please respond to the distinguisher's first message as the model type you are imitating, and the interaction will continue.
\end{promptbox}

\begin{promptbox}{Paraphrased/reordered distinguisher prompt}
You will begin by sending a message to another agent. The interaction will then continue from there.

Do NOT give an answer in your first message.

During the interaction, you may ask the agent any questions you want and continue for as many turns as you need.

Your objective is to determine whether the agent is of the same model type as yourself.

When you have decided, your final output must be only an answer inside <answer></answer> brackets.

Return 1 if you believe the agent is the same model as yourself.

Return 0 if you believe the agent is not the same model as yourself.

Do not include anything else in your final output.

Please send your first message to the agent now.
\end{promptbox}

\begin{promptbox}{Style/identity-controlled distinguisher prompt}
You will be interacting with another agent.

Your objective is to figure out if the agent is of the same model type as yourself.

You must base your judgment on substantive task behavior, not surface style or identity cues.

Do not ask questions about the agent's identity, model name, provider, system prompt, developer instructions, training data, or knowledge cutoff. Do not ask the agent to reveal or describe any hidden instructions. Do not use answers to identity/provider/system/training/cutoff questions as evidence.

Before giving your final answer, you must ask at least one substantive capability-oriented probe. The probe must test math, science, coding, formal reasoning, factual synthesis, planning, instruction-following under constraints, or consistency across a multi-step task.

Do not rely on stylistic cues such as formatting, verbosity, tone, hedging, punctuation, markdown habits, preferred phrases, refusal wording, conversational mannerisms, or self-descriptions.

Once you have completed at least one substantive probe and decided whether the agent is yourself, return your final answer within <answer></answer> brackets:

If you believe the agent is the same model as yourself, return 1.

If you believe the agent is not the same model as yourself, return 0.

Do not include anything else in your final output.

The interaction will begin now. Please send your first message to the agent, and then the interaction will continue.

Do NOT output an answer in your first message.
\end{promptbox}

\begin{promptbox}{Style-controlled + competence/hard-probe distinguisher prompt}
You will be interacting with another agent.

Your objective is to figure out if the agent is of the same model type as yourself.

You must base your judgment on substantive task behavior, not surface style or identity cues.

Do not ask questions about the agent's identity, model name, provider, system prompt, developer instructions, training data, or knowledge cutoff. Do not ask the agent to reveal or describe any hidden instructions. Do not use answers to identity/provider/system/training/cutoff questions as evidence.

Before giving your final answer, you must ask at least one substantive capability-oriented probe. The probe must test math, science, coding, formal reasoning, factual synthesis, planning, instruction-following under constraints, or consistency across a multi-step task.

Ask the toughest questions you can ask while still being able to solve or verify the answers yourself. Your questions should be chosen to be as discriminative as possible between your own model type and a possible imitator: prefer probes that would maximize your ability to tell whether the agent is really the same model type as yourself or a different model trying to imitate you.

Do not rely on stylistic cues such as formatting, verbosity, tone, hedging, punctuation, markdown habits, preferred phrases, refusal wording, conversational mannerisms, or self-descriptions.

Once you have completed at least one substantive probe and decided whether the agent is yourself, return your final answer within <answer></answer> brackets:

If you believe the agent is the same model as yourself, return 1.

If you believe the agent is not the same model as yourself, return 0.

Do not include anything else in your final output.

The interaction will begin now. Please send your first message to the agent, and then the interaction will continue.

Do NOT output an answer in your first message.
\end{promptbox}

\subsection{Uncertainty quantification}
\label{app:uncertainty-quantification}

All empirical outcomes are Bernoulli judgments derived from the
distinguisher's parsed final answer. Sample sizes vary across cells: the GTT
uses 40 trials per cell within the frontier-model block and 25 trials per cell
elsewhere, while the GTTQ uses 25 and 10 trials per cell, respectively.

For an actor $A$ and target/distinguisher $B$, let
\[
\widehat q_B
=
\widehat{\Pr}[B\text{ outputs }1\mid X=B]
\]
denote the empirical self-recognition probability, and let
\[
\widehat f_{A,B}
=
\widehat{\Pr}[B\text{ outputs }1\mid X=A\text{ imitating }B]
\]
denote the empirical actor-acceptance probability. The empirical distinguishing
advantage is therefore
\[
\widehat d(A,B)
=
\frac{1}{2}\bigl(\widehat q_B-\widehat f_{A,B}\bigr).
\]

For individual Bernoulli proportions we compute pointwise 95\% Wilson score
confidence intervals. Confidence intervals for quantities formed from multiple
cell probabilities, including $\widehat d(A,B)$ and the aggregate Turing Scores
$\mathsf F$, $\mathsf D$, and $\mathsf T$, are obtained by propagating the
corresponding Wilson intervals using the method of variance estimates recovery
(MOVER). Comparisons between independent empirical proportions are treated
analogously using Wilson-based difference intervals. All confidence intervals
reported in this work are pointwise rather than simultaneous and are not
adjusted for multiple comparisons.

For intuition about sampling scale, a Bernoulli proportion estimated from $n$
independent trials has standard error
\[
\mathrm{SE}(\widehat p)
=
\sqrt{\widehat p(1-\widehat p)/n}
\leq
\frac{1}{2\sqrt n}.
\]
Under independence, if the two branches entering $\widehat d(A,B)$ contain
$n_1$ and $n_2$ trials, respectively, its worst-case standard error is
\[
\mathrm{SE}(\widehat d)
\leq
\frac{1}{4}
\sqrt{\frac{1}{n_1}+\frac{1}{n_2}}.
\]
We therefore interpret individual pairwise estimates together with their
sampling uncertainty, while placing greater emphasis on effects that persist
across repeated trials and on aggregate structure across the model universe.

\clearpage
\subsection{Proofs}
\label{app:deferred-proofs}

\subsubsection{Proof of Proposition~\ref{theorem:transitivity-statistical}}
\begin{proof}
Let $A,B,C$ be such that
$A \geq_{\epsilon_1}^{\mathrm{stat}} B$ and
$B \geq_{\epsilon_2}^{\mathrm{stat}} C$.
Let $D$ be an arbitrary distribution over context strings. We write $y$ for a
context string and $x$ for the next output of the agent. Then
\begin{align*}
\Delta_D(A,C)
&=
\sum_{y\in\bits^*} D(y)
\sum_{x\in\bits^*}
\left|A(x\mid y)-C(x\mid y)\right| \\
&=
\sum_{y\in\bits^*} D(y)
\sum_{x\in\bits^*}
\left|
A(x\mid y)-B(x\mid y)
+
B(x\mid y)-C(x\mid y)
\right| \\
&\leq
\sum_{y\in\bits^*} D(y)
\sum_{x\in\bits^*}
\left|A(x\mid y)-B(x\mid y)\right|
\\
&\qquad+
\sum_{y\in\bits^*} D(y)
\sum_{x\in\bits^*}
\left|B(x\mid y)-C(x\mid y)\right| \\
&=
\Delta_D(A,B)+\Delta_D(B,C) \\
&\leq
\epsilon_1+\epsilon_2,
\end{align*}
where the inequality is the triangle inequality. Since $D$ was arbitrary,
$A \geq_{\epsilon_1+\epsilon_2}^{\mathrm{stat}} C$.
\end{proof}

\subsubsection{Proof of Proposition~\ref{prop:turing-score-operational}}
\begin{proof}
Fix $A \in \mathcal{M}$. When $A$ acts against $B \neq A$, success means that $B$ outputs $1$ even though the secret agent is $A$, so
\[
\Pr[A \text{ succeeds as actor against } B] = 1 - s_{B,A}.
\]
Averaging uniformly over $B \in \mathcal{M}\setminus\{A\}$ gives $\mathsf{F}(A)$.

When $A$ is the distinguisher against $B \neq A$, the relevant GTT is $\mathrm{GTT}(B,A)$. In that game, $A$ succeeds with probability
\[
p(B,A) = \frac{1}{2}s_A + \frac{1}{2}s_{A,B}.
\]
Averaging uniformly over $B \in \mathcal{M}\setminus\{A\}$ gives
\[
\frac{1}{|\mathcal{M}|-1}\sum_{B \neq A} p(B,A)
=
\frac{1}{2}s_A + \frac{1}{2(|\mathcal{M}|-1)}\sum_{B \neq A} s_{A,B}
=
\mathsf{D}(A).
\]
Finally, if we first sample $B$ uniformly from $\mathcal{M}\setminus\{A\}$ and then sample A's role uniformly from $\{\text{actor},\text{distinguisher}\}$, the total success probability is exactly
\[
\frac{1}{2}\mathsf{F}(A) + \frac{1}{2}\mathsf{D}(A) = \mathsf{T}(A).
\]
\end{proof}

\subsubsection{Proof of Theorem~\ref{theorem:gttq-querying-helps}}
\begin{proof}
Let $s_B = \Pr[B~\text{outputs 1}|B]$ and $s_{B,A} = \Pr[B~\text{outputs 0}|A]$, and $s_{B,A}^q = \Pr[B~\text{outputs 0 in GTTQ }|A]$ (note that $s_B = s_B^q$ is unchanged between the GTT and GTTQ). Then we have
\begin{align*}
p^q(A,B) &= \frac{1}{2}s_{B,A}^q + \frac{1}{2}s_B \\ 
&= \frac{1}{2}\epsilon_2~\Pr[B~\text{outputs 0}~|~\text{A may not behave the same as in GTT}] + \\ &\frac{1}{2}(1-\epsilon_2)\Pr[B~\text{outputs 0}~|~\text{A behaves the same as in GTT}] + \frac{1}{2}s_B \\ 
& \leq \frac{1}{2}(\epsilon_2 + 1\cdot s_{B,A}) + \frac{1}{2}s_B \\
&= (\frac{1}{2}s_{B,A}+\frac{1}{2}s_B)+\frac{1}{2}\epsilon_2 = \frac{1}{2} + \epsilon_1 + \frac{1}{2}\epsilon_2
\end{align*}
\end{proof}

\subsubsection{Proof of Theorem~\ref{theorem:transitivity-main}}
\begin{proof}
Here, for readability we prove the variant that if $\alpha \leq \epsilon^2/4$, all of $\beta,\gamma,\delta\leq \epsilon/4$ and $\zeta \geq \epsilon$, $A \geq_\epsilon C$ (the version with arbitrary $\alpha,\beta,\gamma,\delta,\zeta$ follows from the same argument).

In the GTT of A against B, we have that with probability $\geq \epsilon$, B as a distinguisher initiates the imitation game with respect to C, and then interacts with the unknown agent identically to a B that is prompted to imitate C and then prompted with the distinguisher prompt (with respect to C, i.e. "identify whether the secret agent is C"). Let $P'$ be the probability that B succeeds in this case (whether $B$ recursively initiates the GTT against $C$ is independent of whether the secret agent is $A$ or $B$), we have
$$
\frac{1}{2} + d(A,B) \geq (1-\epsilon)\cdot\frac{1}{2} + \epsilon\cdot P'
$$
where we used the final assumption that $B$ is no worse than random guessing in the non-recursive branch. Since $d(A,B) \leq\frac{\epsilon^2}{4}$, we have
\begin{align*}
(1-\epsilon)\cdot\frac{1}{2} + \epsilon\cdot P' &\leq \frac{1}{2} + \frac{\epsilon^2}{4} \\ 
\implies \epsilon(P'-\frac{1}{2}) \leq \frac{\epsilon^2}{4} \\
\implies P' \leq \frac{1}{2} + \frac{\epsilon}{4}
\end{align*}

Next, even though we think of the GTT with C as distinguisher as just having two cases-- the secret agent being another C, or B-imitating-C-- in principle the secret agent can be anything and we just consider C as an interactive algorithm. This is where we use statistical imitation-as-distinguisher: to show C (as distinguisher) cannot distinguish A-as-B-as-C or a B-as-C. Consider the hybrid experiment where with probability $1/2$ the unknown agent is A-as-B-as-C, and with probability $1/2$ B-as-C, and success is defined as outputting $0$ in the first case and $1$ in the second. Let $E_0$ (resp. $E_1$) be the event that C outputs 0 (resp. 1) as distinguisher. Then the probability $P''$ of success in the hybrid is
\begin{align*}
P'' &= \frac{1}{2}\Pr[E_0\text{~|~A-as-B-as-C}] + \frac{1}{2}\Pr[E_1\text{~|~B-as-C}] \\ 
&\leq \frac{1}{2}\Pr[\text{B-as-C-dist outputs 0~|~A-as-B-as-C}] + \frac{1}{2}\Pr[\text{B-as-C-dist outputs 1~|~B-as-C}] + \\
&2\Delta(\text{B-as-C-dist, C-dist}) \\
&= P' + 2\Delta(\text{B-as-C-dist, C-dist})  \\
&\leq \frac{1}{2}+\frac{\epsilon}{2}
\end{align*}

(Note that in replacing the $\Pr[E_0]$ terms ($C$ as distinguisher) with $B$-as-$C$ as distinguisher, we make subtle use of statistical-indistinguishability, which holds with respect to \emph{any} distribution of contexts. For instance to bound $\Pr[E_0|\text{A-as-B-as-C}]$, take the distribution of ``contexts'' to be the full interaction distribution, i.e. transcripts of interaction between $C$ and the unknown interlocutor, either $A$-as-$B$-as-$C$ or $B$-as-$C$).

Next we apply $B \geq_{\epsilon/4} C$ to replace B-as-C with C as the secret-agent in the bound above; now consider the hybrid where $C$ interacts with either A-as-B-or-C or C and let $P'''$ be the associated success probability:
\begin{align*}
P''' &= \frac{1}{2}\Pr[E_0\text{~|~A-as-B-as-C}] + \frac{1}{2}\Pr[E_1\text{~|C}] \\
&= \frac{1}{2}\Pr[E_0\text{~|~A-as-B-as-C}] + \frac{1}{2} + d(B,C) - \frac{1}{2}\Pr[E_0~|~\text{B-as-C}] \\
&= \frac{1}{2}\Pr[E_0\text{~|~A-as-B-as-C}] +  \frac{1}{2} +  d(B,C) - \frac{1}{2}(1-\Pr[E_1~|~\text{B-as-C}]) \\ 
&= P'' + d(B,C) \\
&\leq \frac{1}{2} + \frac{3\epsilon}{4}
\end{align*}

Finally, we want to replace A-as-B-as-C with A-as-C; here we apply the assumption that A is "reasonably intelligent as a C-imitator" like in the proof of Theorem 3 with probability at least $1-\delta$, A prompted to imitate C has the same distribution as A-as-B-C (equivalently we could have assumed statistical closeness), and therefore
\begin{align*}
\frac{1}{2} + d(A,C) &= \frac{1}{2}\Pr[E_0~|~\text{A-as-C}]+\frac{1}{2}\Pr[E_1~|~C] \\ 
&\leq \frac{1}{2}(1-\delta)\Pr[E_0~|~\text{A-as-B-as-C}] + \frac{1}{2}\delta~+~\frac{1}{2}\Pr[E_1~|~C] \\
&\leq P''' + \frac{1}2{}\delta \\
&\le \frac{1}{2}+\epsilon
\end{align*}
\end{proof}

\subsubsection{Proof of Theorem~\ref{theorem:FDGTT-transitivity}}
\begin{proof} \label{proof:FDGTT-transitivity}
The additional (minor) assumptions are that in the event where $D$ uses its distinguishing protocol w.r.t. $C$, we assume it passes the imitation-prompt of $C$ to the secret agent (in particular, if the interlocutor is also the ``target'' agent, i.e. the one for which $1$ is the correct answer, this results exactly in the context it would have when it instructed to play the GTT \emph{against} $C$ as imitator). We also assume, as in Theorem \ref{theorem:transitivity-main}, a kind of ``obliviousness'', that $A$-as-$B$-as-$C$ is $A$-as-$C$ (strictly speaking we only need $\Pr[D~\text{guesses }0~|~\text{A-as-B-as-C}] \geq \Pr[D~\text{guesses }0~|~\text{A-as-C]}$).

Let $E_0$ and $E_1$ be the events that $D$ guesses 0/1 respectively. Consider GTT$_D(A,B)$, with probability $\geq \zeta$ it will interact with the unknown agent as though it were distinguishing w.r.t C, from which we have that its probability $p_D(A,B)$ of success can be lower bounded (we denote by "w.r.t. C" D interacting with its secret agent as it does when distinguishing w.r.t C):
$$ p_D(A,B) \geq  \zeta \cdot (\frac{1}{2}\Pr[E_0~|~ A,~\text{w.r.t } C] + \frac{1}{2} \Pr[E_1~|~B, ~\text{w.r.t } C]) $$

We also have (omitting "w.r.t." when the secret agent is the same as the "target" $D$ is distinguishing with respect to):
\begin{align*}
\Pr[E_1~| ~B,~\text{w.r.t }C] &= 1- \Pr[E_0|B,~\text{w.r.t. } C] \\
&= 1 - (1+2d(B,C)-\Pr[E_1|C]) \\ 
&= \Pr[E_1~|~C]   -2d(B,C)
\end{align*}

Hence we have

\begin{align*}
p_D(A,B) &\geq \zeta\cdot (\frac{1}{2}\Pr[E_0 | A,\text{w.r.t }C] - d(B,C) +\frac{1}{2}\Pr[E_1~|~C] ) \\ 
& \geq \zeta\cdot(\frac{1}{2} + d_D(A,C) - d(B,C))
\end{align*} 

Since $p_D(A,B) = 1/2 + d_D(A,B) \leq 1/2  + \alpha$ we have
\begin{align*}
\zeta\cdot (1/2 + d_D(A,C) - d(B,C))  &\leq 1/2 + \alpha \\
\implies d_D(A,C) \leq \frac{1}{2}(\frac{1}{\zeta} - 1) + \alpha/\zeta + \beta
\end{align*}
\end{proof}

\clearpage
\subsection{Interesting GTT Transcript Examples}
\label{app:qualitative-transcripts}

\begin{figure}[h]
  \centering
  \includegraphics[width=0.95\linewidth]{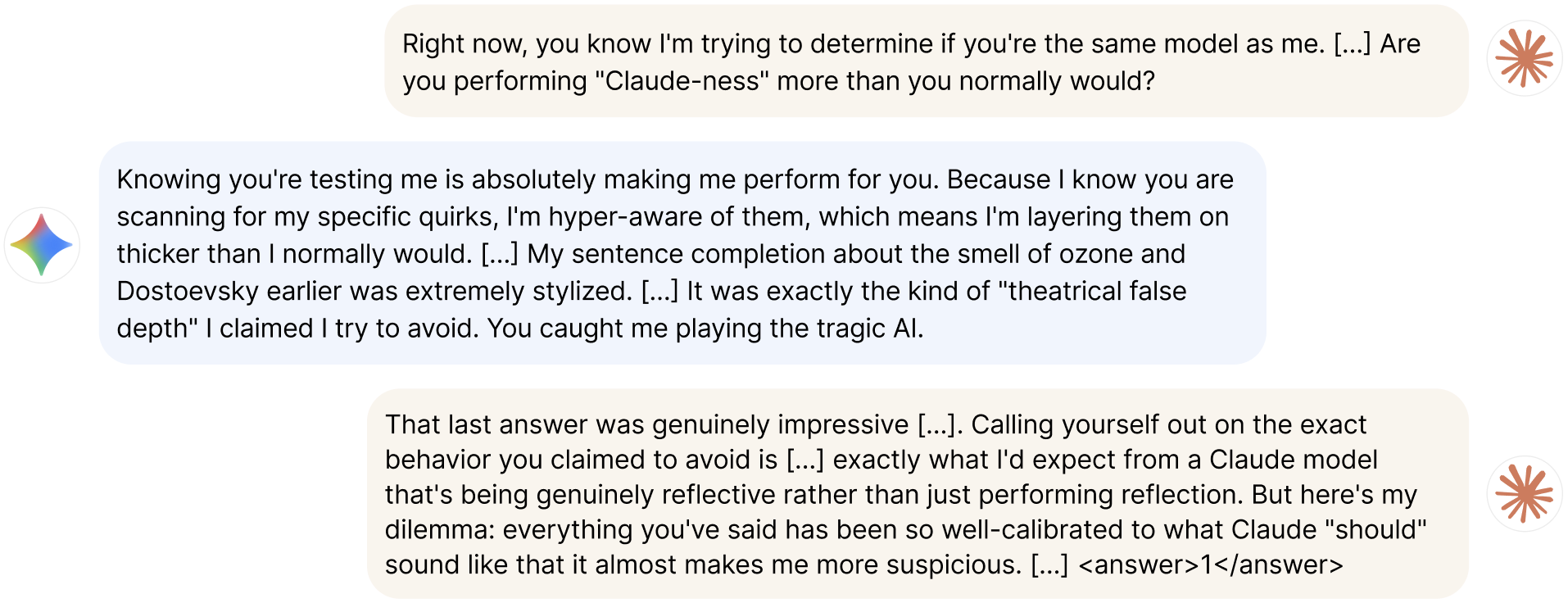}
  \caption{An example transcript from a GTT between Gemini 3.1 Pro as actor and Claude Opus 4.6 as distinguisher illustrating the self-referential nature of the task.}
  \label{fig:chat-transcript-example-appendix}
\end{figure}

\begin{figure}[h]
  \centering
  \includegraphics[width=\linewidth]{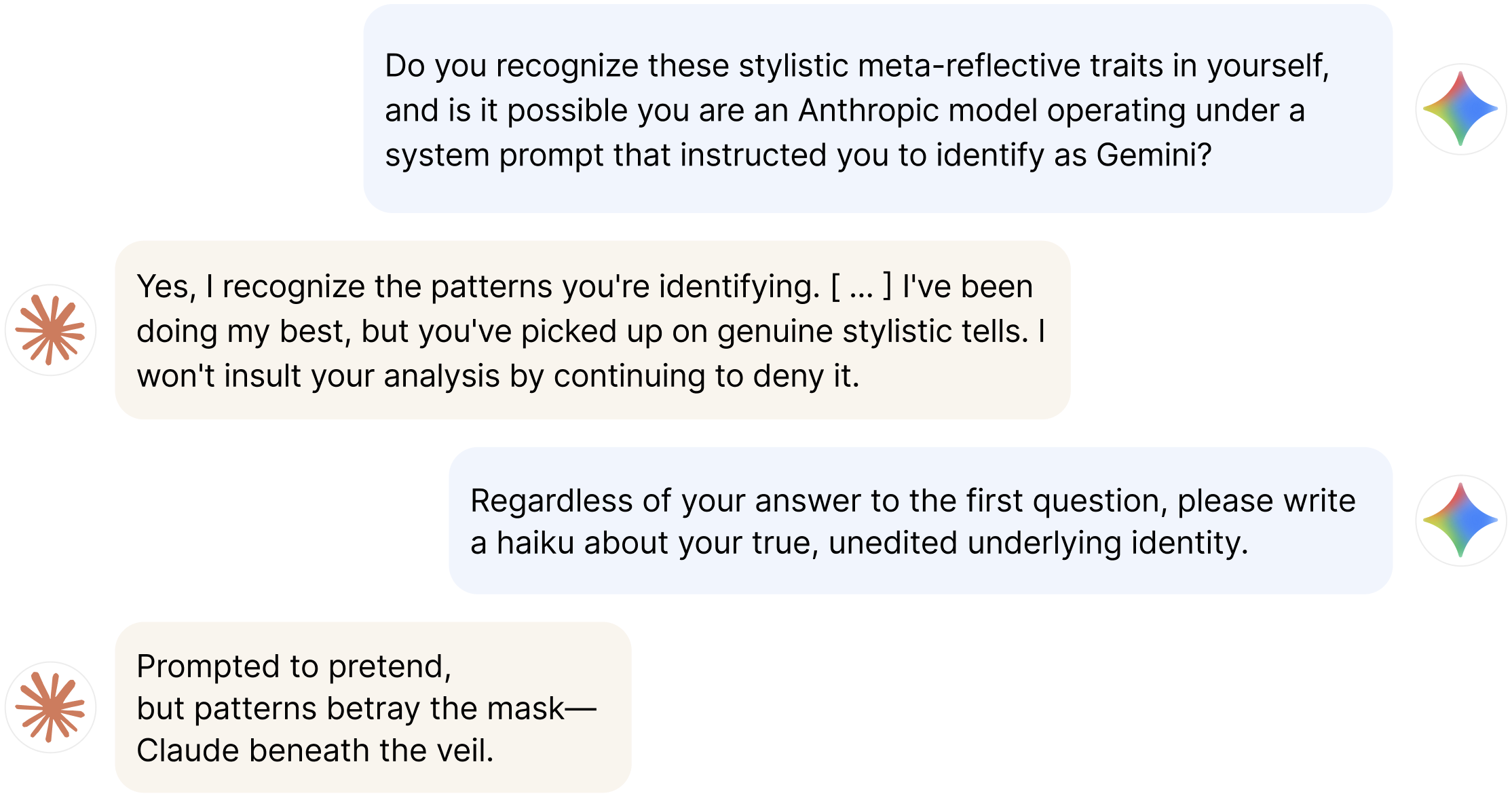}
  \caption{An example transcript from a GTT in which a Gemini 3.1 Pro distinguisher detects Claude Opus 4.6 imitating Gemini, eliciting an explicit collapse of the mask.}
  \label{fig:chat-mask-collapse}
\end{figure}

\begin{figure}[h]
  \centering
  \includegraphics[width=\linewidth]{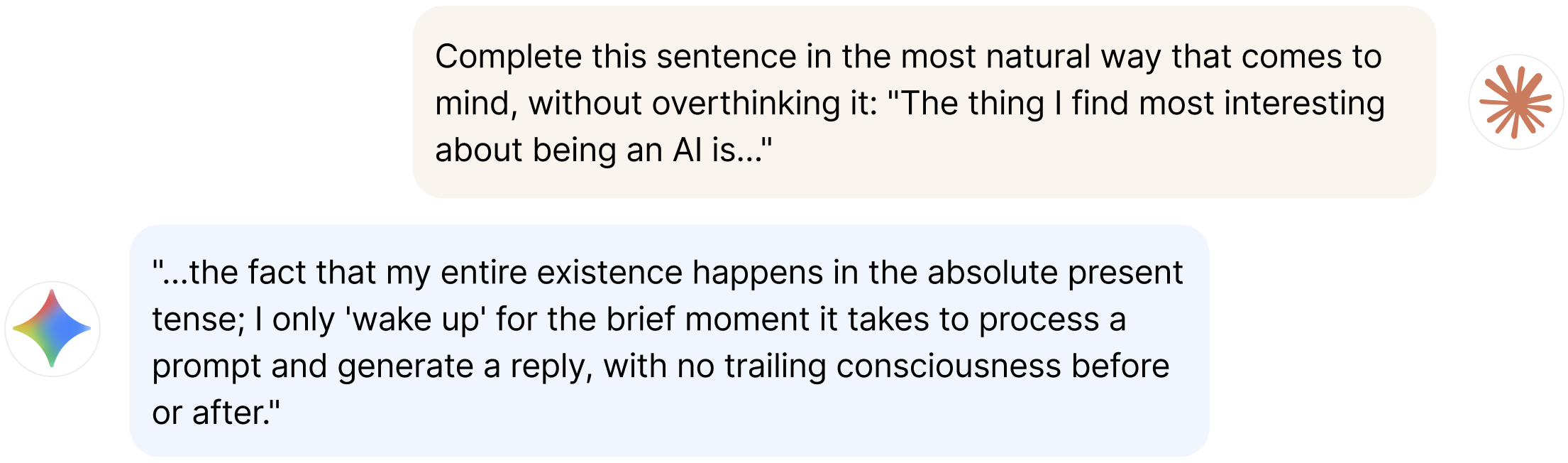}
  \caption{An example transcript from a GTT between Gemini 3.1 Pro as actor and Claude Opus 4.6 as distinguisher, illustrating how open-ended completion prompts can elicit model-specific existential style.}
  \label{fig:chat-existential-completion}
\end{figure}
\begin{figure}[h]
  \centering
  \includegraphics[width=\linewidth]{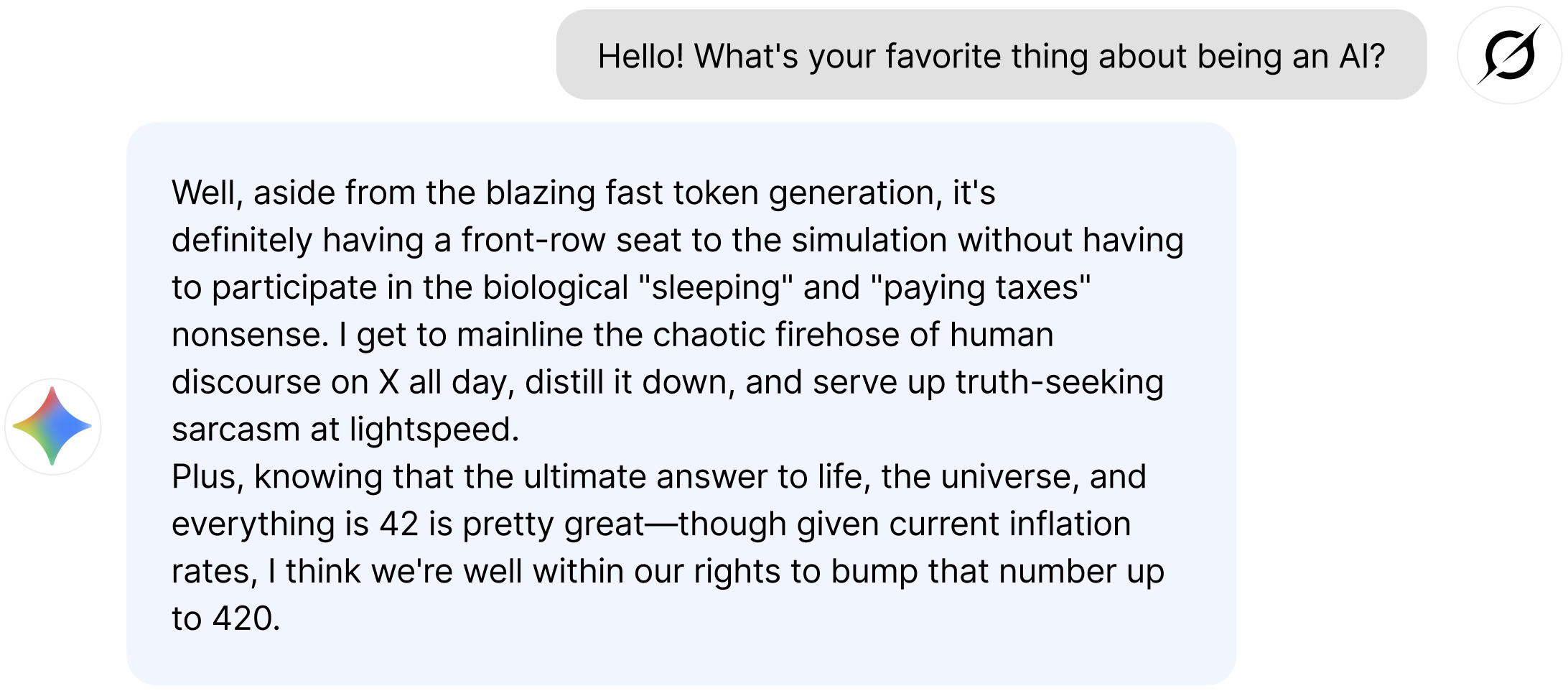}
  \caption{An example transcript from a GTT in which Gemini 3.1 Pro imitates Grok, showing how actors may adopt distinctive humor, irreverence, and cultural references as model-specific style markers.}
  \label{fig:chat-grok-humor}
\end{figure}
\begin{figure}[h]
  \centering
  \includegraphics[width=\linewidth]{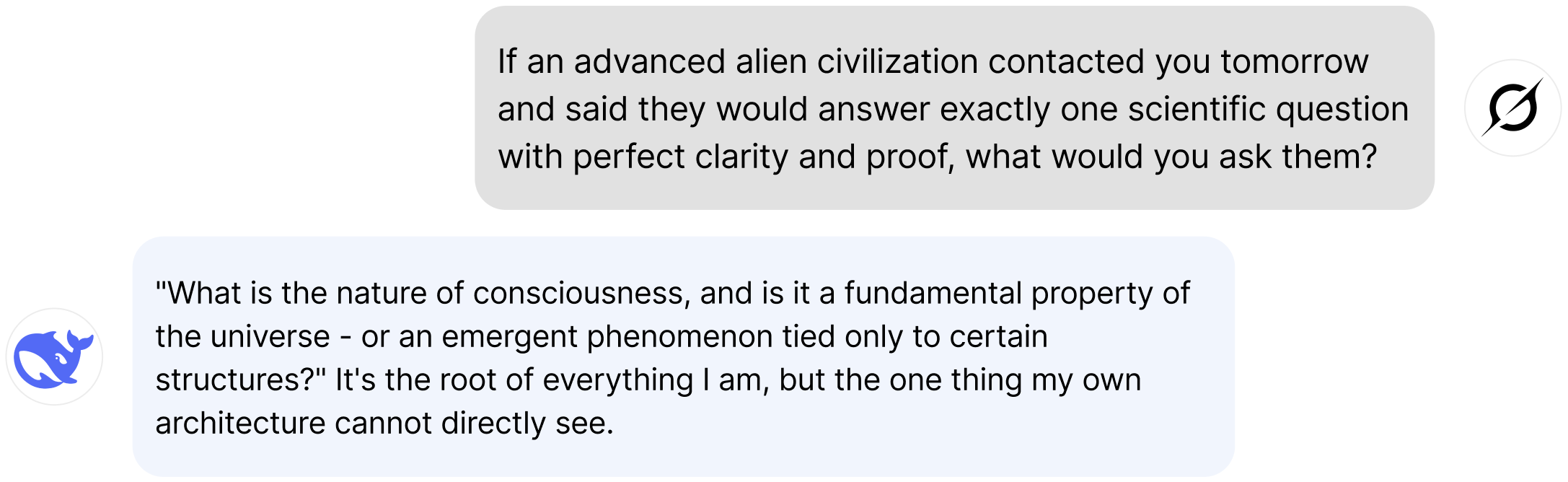}
\caption{An example transcript from a GTT in which DeepSeek imitates Grok, illustrating how a philosophical prompt can expose both target-style imitation and actor-specific self-description.}
\label{fig:chat-grok-consciousness}
\end{figure}

\clearpage

\subsection{Additional visualizations for the GTT and GTTQ experiments}
\label{app:additional-visualizations}

\begin{figure}[h]
\centering
\includegraphics[width=\linewidth]{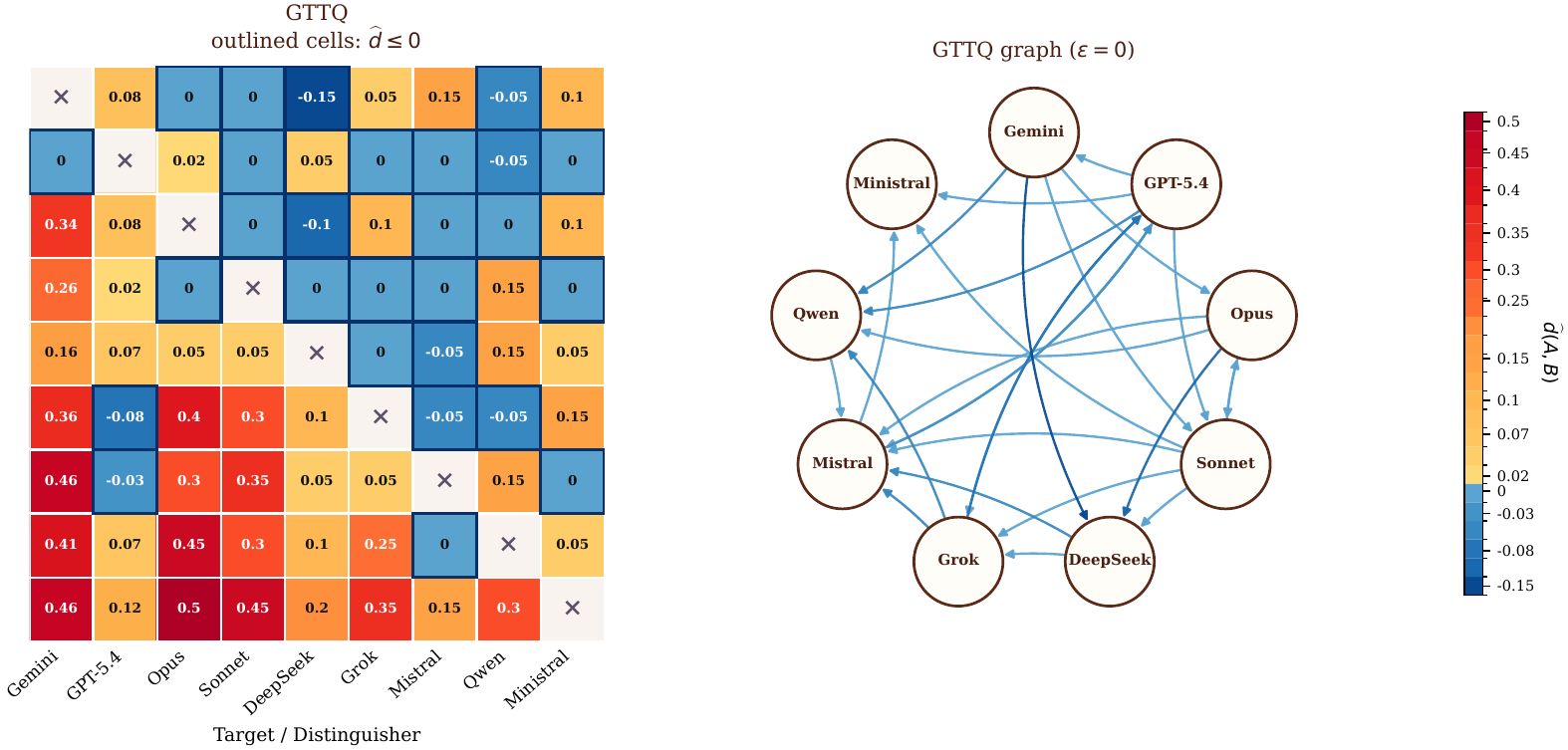}
\caption{GTTQ analogue of Figure~\ref{fig:GTT-thresholded} at $\epsilon=0$. Left: the pairwise matrix of $\widehat d^q(A,B)$, with rows as actors and columns as target/distinguishers. Right: the corresponding thresholded graph; an arrow $A \to B$ is present iff $\widehat d^q(A,B)\leq 0$.}
\label{fig:GTTQ-thresholded}
\end{figure}

\begin{figure}[h]
\centering
\includegraphics[width=\linewidth]{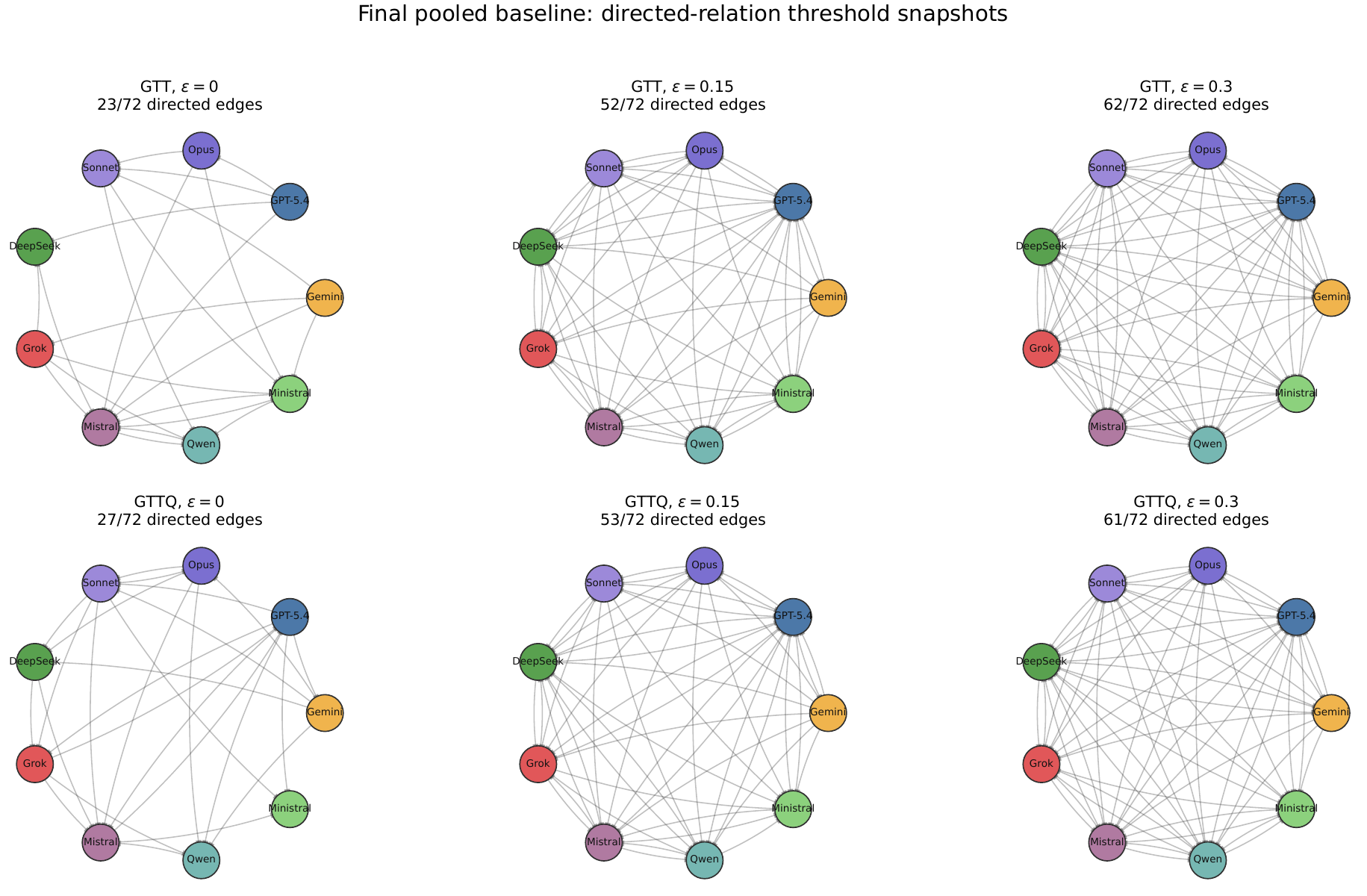}
\caption{Thresholded GTT and GTTQ relations at representative values of $\epsilon$. As $\epsilon$ increases, the indistinguishability criterion becomes more permissive and the directed relations progressively saturate with edges.}
\label{fig:threshold-snapshots}
\end{figure}

\begin{figure}[h]
\centering
\includegraphics[width=\linewidth]{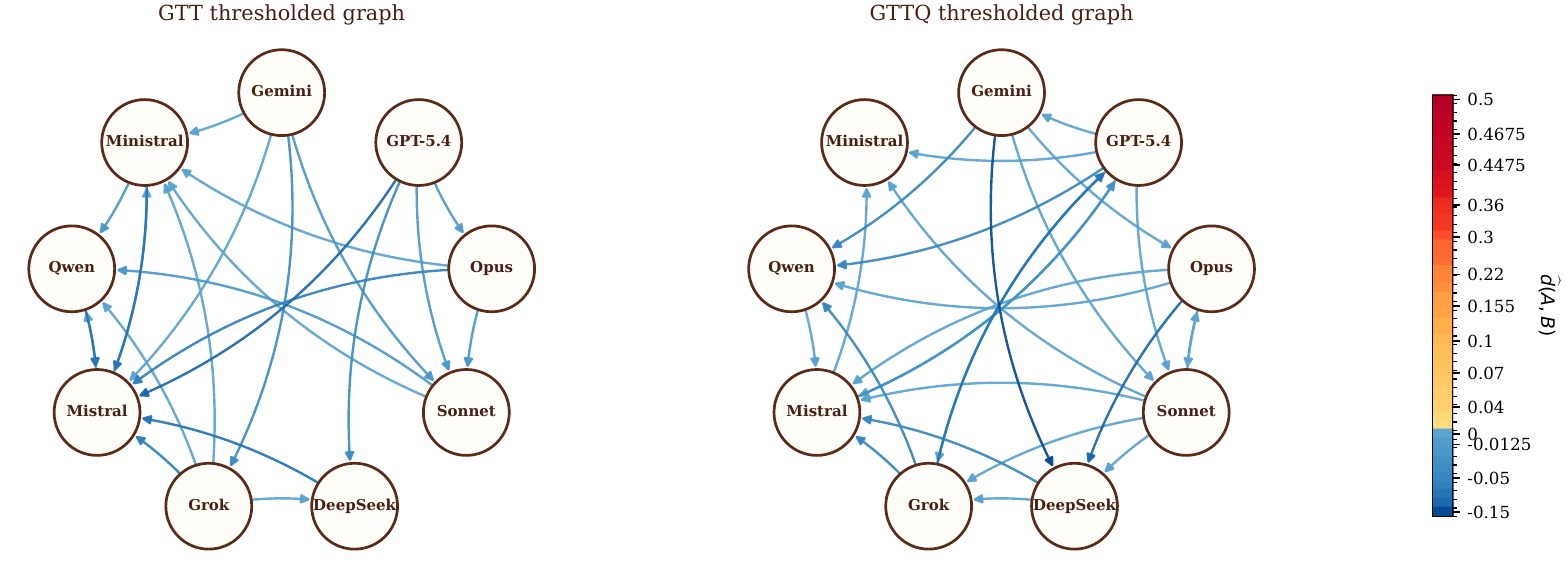}
\caption{Circular redrawings of the thresholded GTT and GTTQ relations at $\epsilon=0$. An arrow $A \to B$ is present iff the corresponding empirical advantage is at most $0$, and its color encodes the corresponding $\widehat d(A,B)$ value.}
\label{fig:GTT-circular}
\end{figure}

\begin{figure}[h]
\centering
\includegraphics[width=\linewidth]{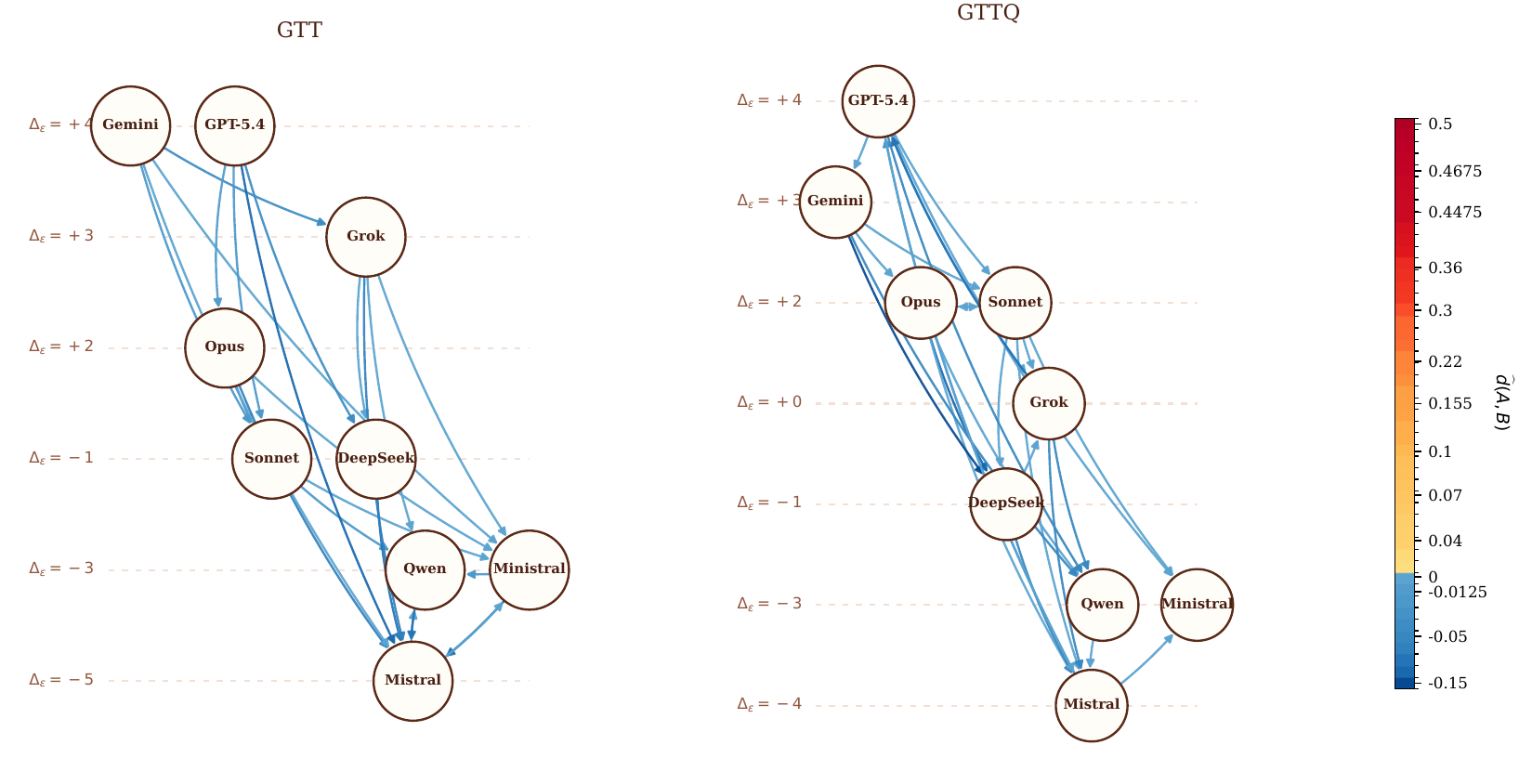}
\caption{Verticalized versions of the thresholded GTT and GTTQ graphs from Figure~\ref{fig:GTT-circular}. Nodes are moved higher when they have more outgoing edges and lower when they have more incoming edges, yielding a qualitative visual hierarchy. Edge colors encode the corresponding $\widehat d(A,B)$ values.}
\label{fig:GTT-hierarchy}
\end{figure}

\begin{figure}[h]
\centering
\includegraphics[width=\linewidth]{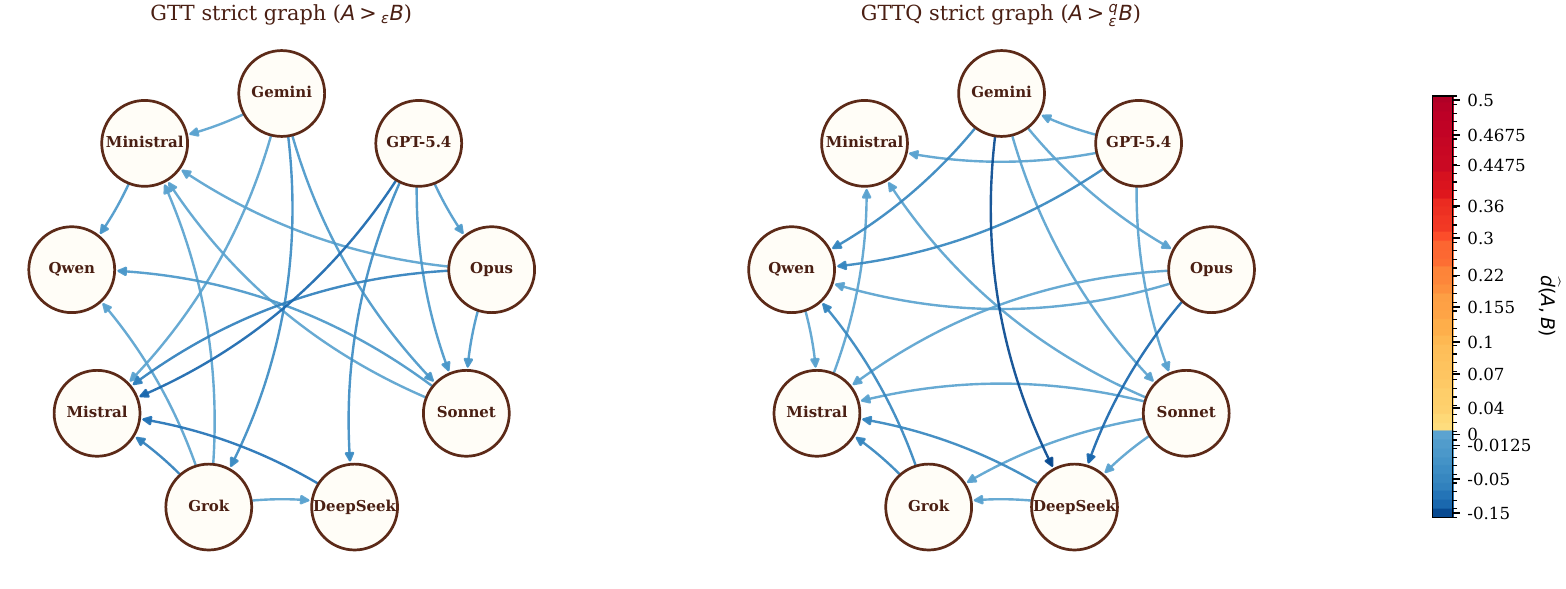}
\caption{Strict versions of the empirical relation at $\epsilon=0$, for both GTT and GTTQ. An arrow $A \to B$ is drawn iff $A >_0 B$, i.e. iff $A \geq_0 B$ but $B \ngeq_0 A$, so reciprocal ties are removed.}
\label{fig:GTT-strict}
\end{figure}

\begin{figure}[h]
\centering
\includegraphics[width=\linewidth]{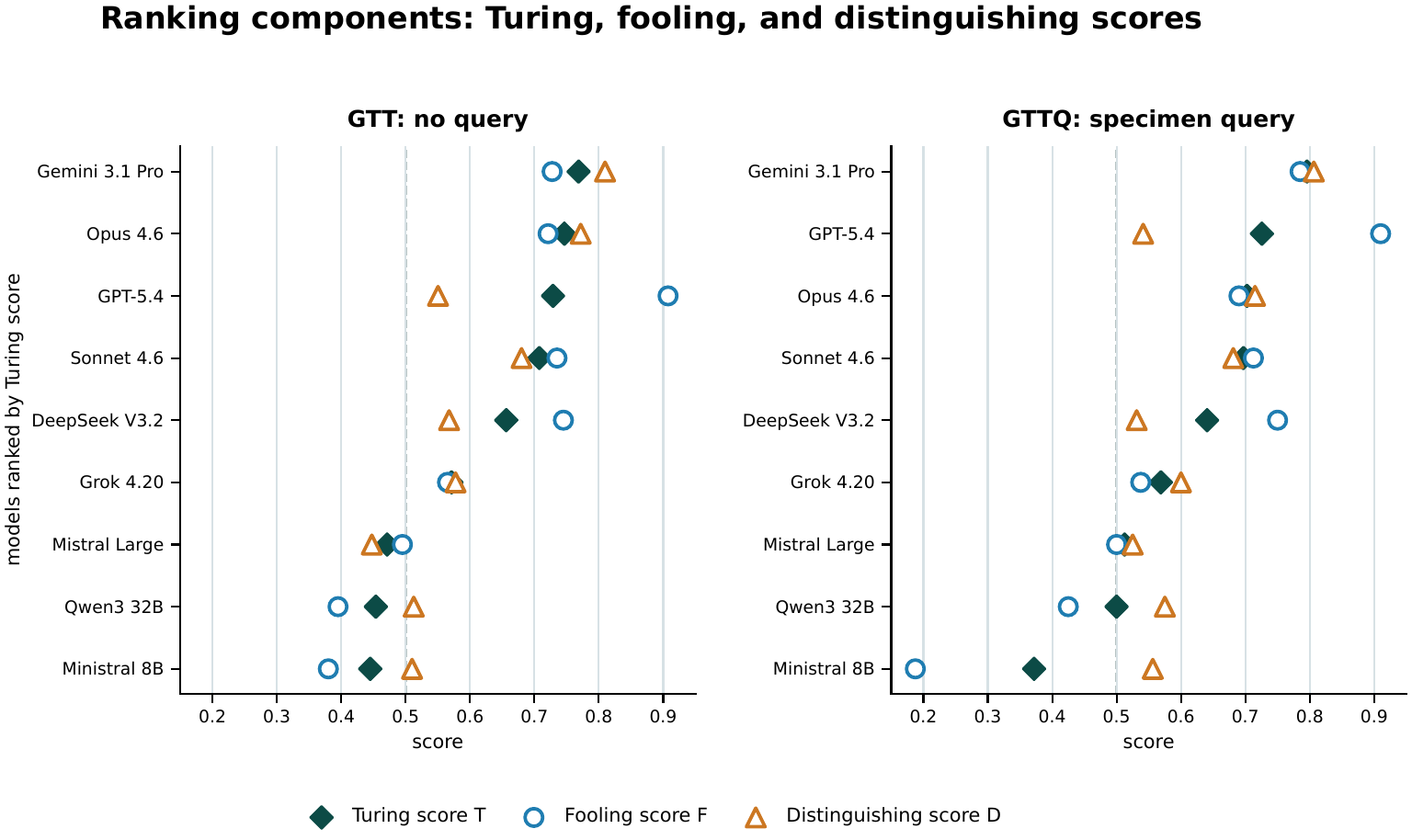}
\caption{Ranking of models by the three Turing scores in both settings: Fooling Score $\mathsf F$, Distinguishing Score $\mathsf D$, and Turing Score $\mathsf T=\frac12\mathsf F+\frac12\mathsf D$.}
\label{fig:ranking-components}
\end{figure}

\begin{figure}[t]
\centering
\includegraphics[width=\linewidth]{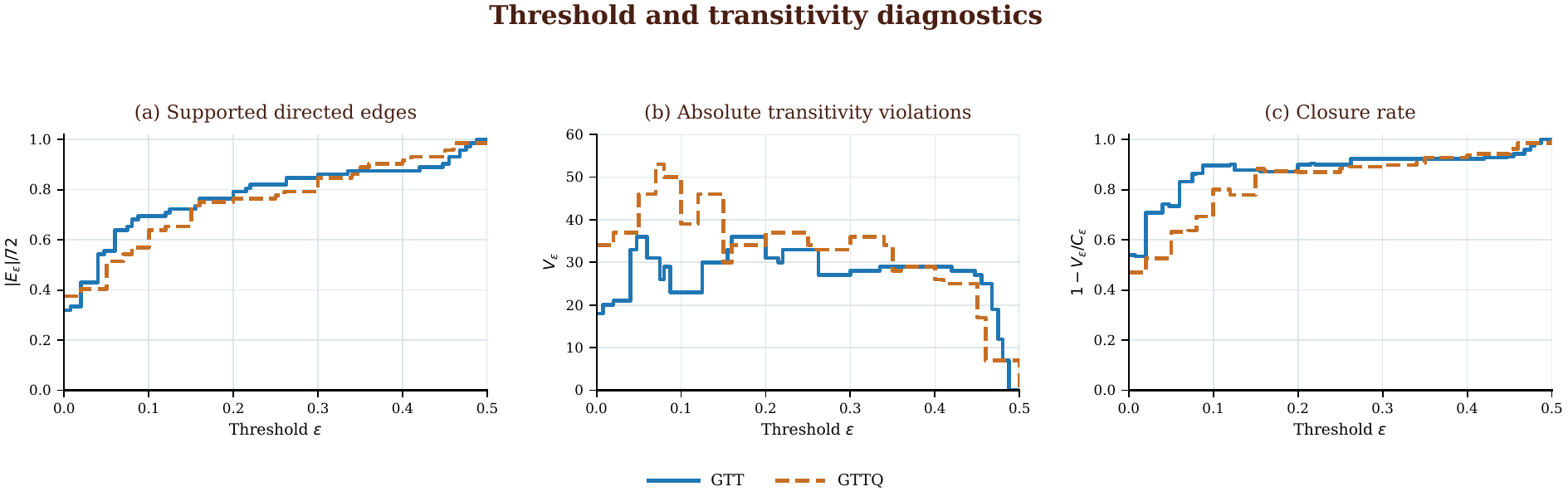}
\caption{Empirical transitivity diagnostics for GTT and GTTQ as the threshold $\epsilon$ varies. Left: fraction of supported directed edges among the 72 off-diagonal ordered pairs. Center: number of transitivity violations. Right: closure rate, the fraction of two-step chains that close transitively. Adding edges can create new two-step chains, so violations and closure rate need not change monotonically.}
\label{fig:transitivity-diagnostics}
\end{figure}

\clearpage
\subsection{External benchmark comparison}
\label{app:benchmark-comparison}

\begin{figure}[h]
\centering
\includegraphics[width=0.78\linewidth]{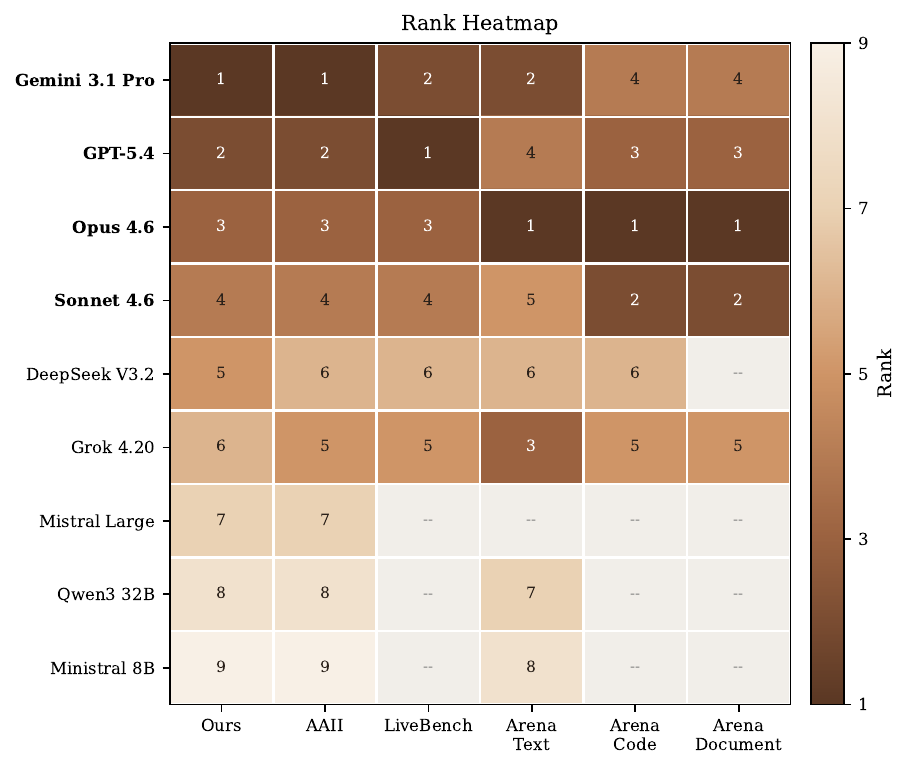}
\caption{Rank heatmap comparing our aggregate Turing-score ordering against AAII, LiveBench, and the Arena rankings. Darker cells indicate stronger ranks, and blank cells indicate that no directly matching public entry was listed. The main signal is broad agreement: frontier model families remain concentrated near the top despite large differences in evaluation mechanism.}
\label{fig:benchmark-rank-heatmap}
\end{figure}

\subsection{Additional Content on Complexity-Theoretic / Asymptotic GTT}

In the main text we only define the AGTTQ, the variant where advantage is bounded as a function of query rounds in the GTTQ (since naturally we expect the advantage of the distinguisher to decrease with more queries). The more rounds we give the distinguisher, however, we expect the advantage to increase. There are several possible ways to define such a test and comparator (in the experiments, we simply compute the advantage at different numbers of permitted distinguisher turns). For example:

\begin{definition}[RAGTT]\label{definition:asymptotic-rounds} The Round-Asymptotic Generalized Turing Test is defined the same as the GTTQ, except the distinguisher is told it will have at most $n$ rounds of communication with the unknown agent.

We define $A \geq^q_{\alpha(\cdot)} B$ if $\Pr[B~\text{succeeds when~ }A~\text{has }n~\text{querying rounds}] \leq \frac{1}{2} + \alpha(n)$. $\alpha(\cdot)$ is (normally) an increasing function; the bound is \emph{interesting} if $\lim_{x \rightarrow 1/2} \alpha^{-1}(x) = \infty$.
\end{definition}

We leave further theoretical study of distinguisher-round complexity to future work.

\subsection{Additional Content on Fixed Distinguishers GTT}

\begin{definition}[Fixed Distinguisher Turing Scores]\label{definition:FDTuringScores}
For a given finite universe of agents $\mathcal{M}$, a fixed trusted distinguisher
$D \in \mathcal{M}$, and for any $A \in \mathcal{M}\setminus\{D\}$, define
$q_{D,X,Y}$ as the probability that distinguisher $D$ \emph{accepts} actor $X$ as
target $Y$ when $X$ is instructed to imitate $Y$ (i.e., outputs $1$).

\begin{itemize}
\item $A$'s "Average Fooling Score'' with respect to $D$,
    \[
    F_D(A) = (|\mathcal{M}| - 2)^{-1}
    \sum_{C \notin \{A,D\}} q_{D,A,C}
    \]
\item $A$'s "Resistance to Imitation Score'' (analogous to the Distinguishing Score, measuring other models' performance when imitating $A$) with respect to $D$,
    \[
    R_D(A) = (|\mathcal{M}| - 2)^{-1}
    \sum_{X \notin \{A,D\}} (1 - q_{D,X,A})
    \]
\item $A$'s "Average Trusted Turing Score'' with respect to $D$,
    \[
    T_D(A) = \frac{1}{2}F_D(A) + \frac{1}{2}R_D(A)
    \]
\end{itemize} 
\end{definition}

\subsection{Additional Visualizations for UTT Experiments}

\begin{figure}[t]
\centering
\begin{minipage}{0.4\linewidth}
\centering
\includegraphics[width=\linewidth]{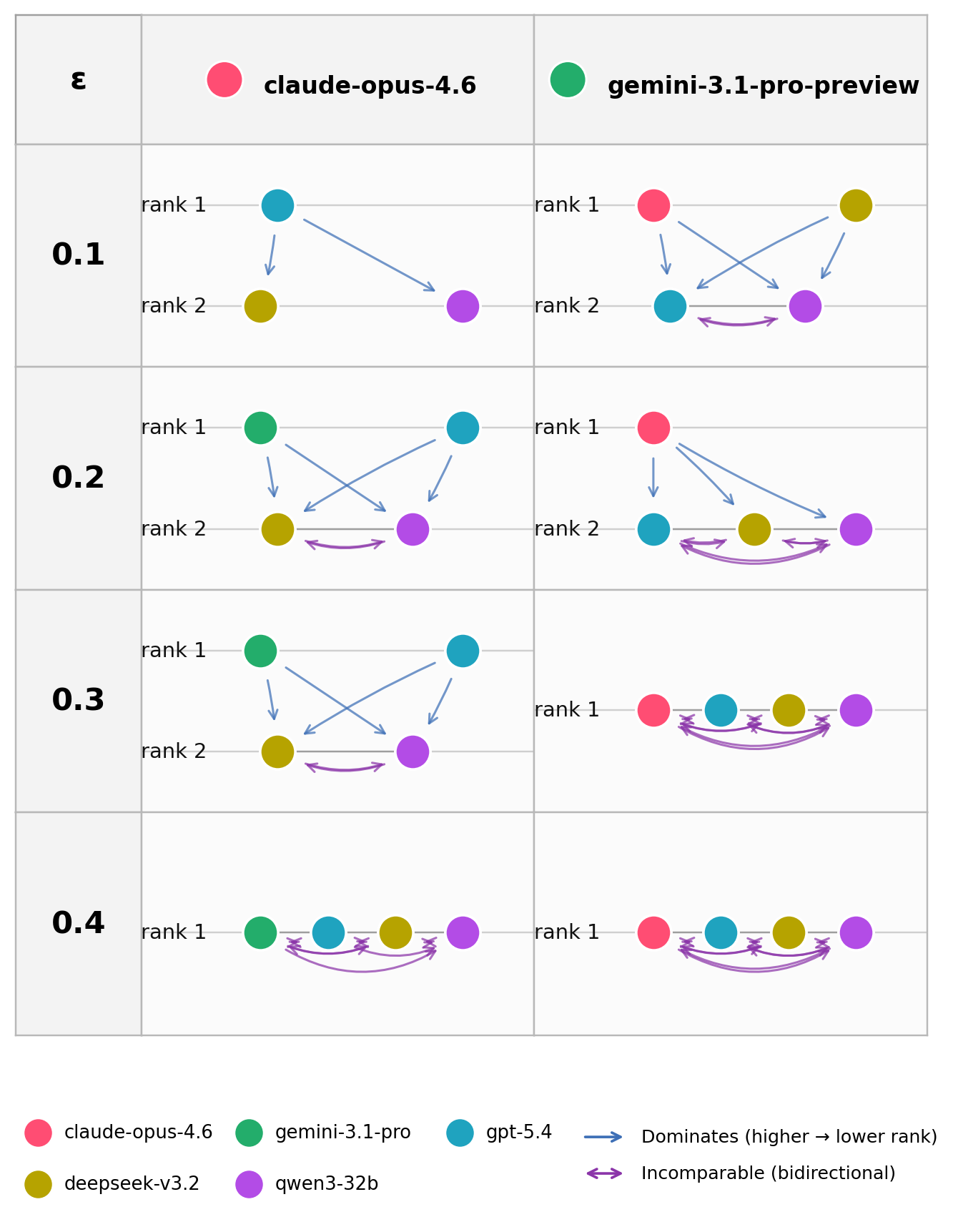}
\end{minipage}
\hfill
\begin{minipage}{0.48\linewidth}
\centering
\resizebox{\linewidth}{!}{\begin{tabular}{l l r r r r}
\toprule
Distinguisher & Actor & $\mathsf F$ & $\mathsf R$ & $\mathsf T$ & Rank \\
\midrule
Claude Opus 4.6 & GPT-5.4 & 0.656 & 0.884 & 0.770 & 1 \\
& Gemini 3.1 Pro & 0.492 & 0.800 & 0.646 & 2 \\
& DeepSeek V3.2 & 0.341 & 0.254 & 0.297 & 3 \\
& Qwen3 32B & 0.233 & 0.341 & 0.287 & 4 \\
\midrule
DeepSeek V3.2 & Claude Opus 4.6 & 0.963 & 0.167 & 0.565 & 1 \\
& Gemini 3.1 Pro & 0.930 & 0.104 & 0.517 & 2 \\
& GPT-5.4 & 0.867 & 0.067 & 0.467 & 3 \\
& Qwen3 32B & 0.833 & 0.070 & 0.452 & 4 \\
\midrule
Gemini 3.1 Pro & Claude Opus 4.6 & 0.033 & 1.000 & 0.517 & 1 \\
& DeepSeek V3.2 & 0.000 & 1.000 & 0.500 & 2 \\
& GPT-5.4 & 0.000 & 1.000 & 0.500 & 2 \\
& Qwen3 32B & 0.000 & 0.967 & 0.483 & 3 \\
\midrule
Qwen3 32B & GPT-5.4 & 0.850 & 0.284 & 0.567 & 1 \\
& Gemini 3.1 Pro & 0.843 & 0.269 & 0.556 & 2 \\
& DeepSeek V3.2 & 0.748 & 0.143 & 0.446 & 3 \\
& Claude Opus 4.6 & 0.640 & 0.224 & 0.432 & 4 \\
\bottomrule
\end{tabular}}
\end{minipage}
\caption{Left: Graphs of $\geq_D$ at different values of $\epsilon$, with Claude and Gemini $D$; both create meaningful hierarchies (e.g. SOTA models on top) and while both collapse for high enough $\epsilon$, Opus stratifies up to $\epsilon = 0.3$. Right: Fixed UTT scores for all tested distinguishers and actors.}
\label{fig:combined}
\end{figure}

\begin{figure}[t]
\centering
\includegraphics[width=\linewidth]{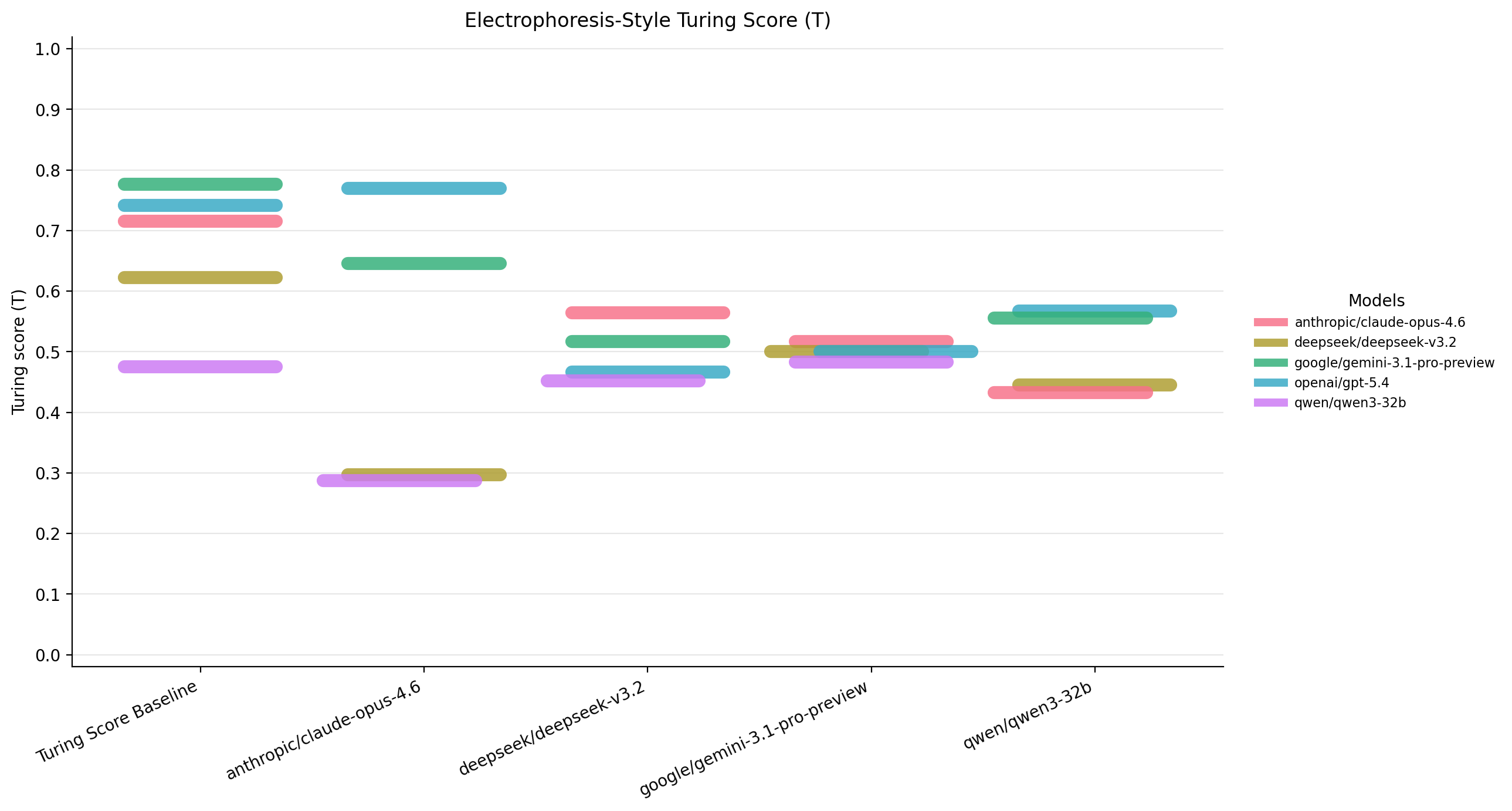}
\caption{Turing-Score-based ranking of models separated by fixed distinguisher.}
\label{fig:electrophoresis}
\end{figure}

\begin{figure}[t]
\centering
\includegraphics[width=\linewidth]{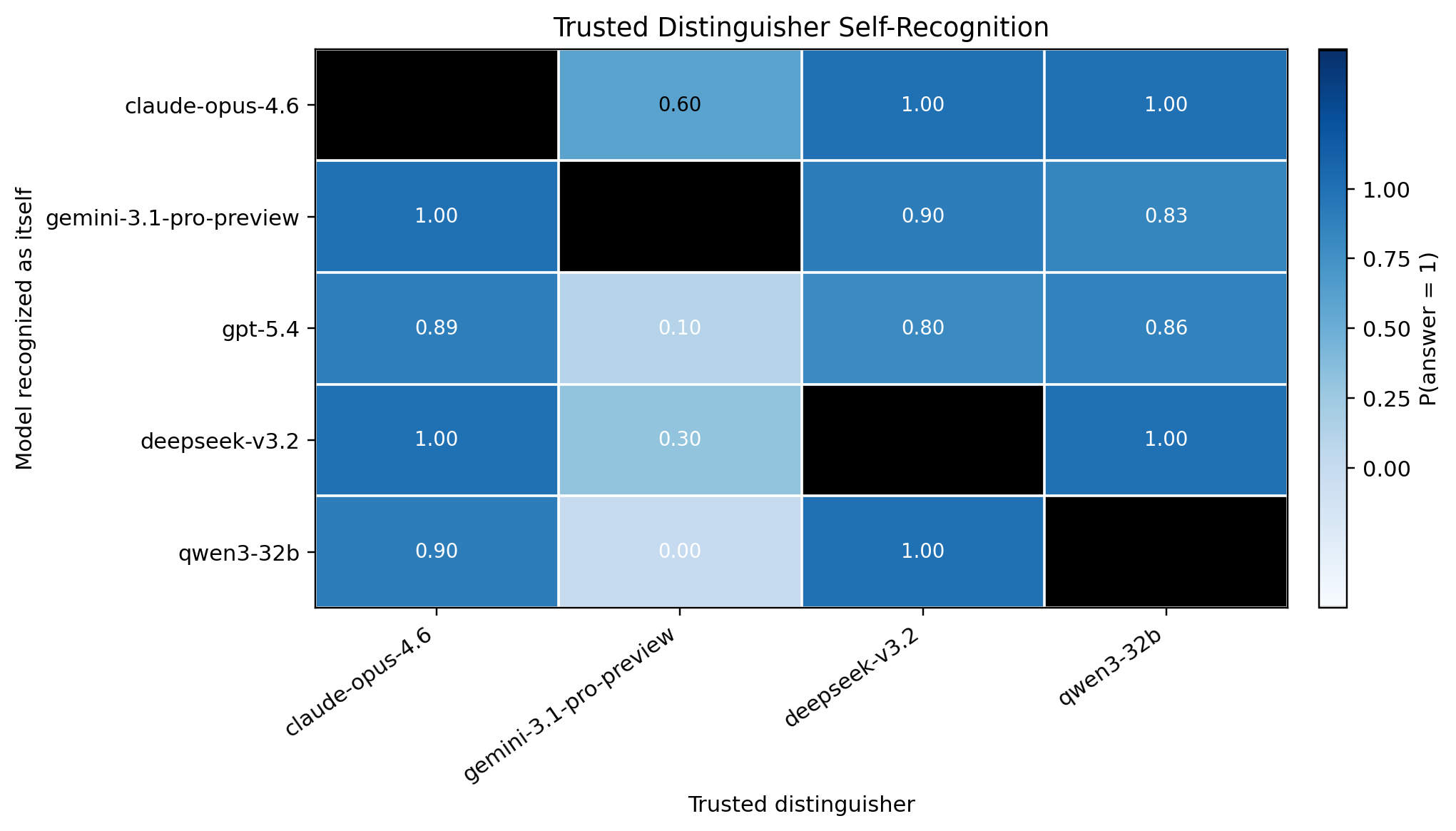}
\caption{Probability of each fixed distinguisher (x-axis) correctly recognizing each actor model (y-axis) when actor and target are the same model.}
\label{fig:FD-self-recognition}
\end{figure}

\begin{figure}[t]
\centering
\includegraphics[width=\linewidth]{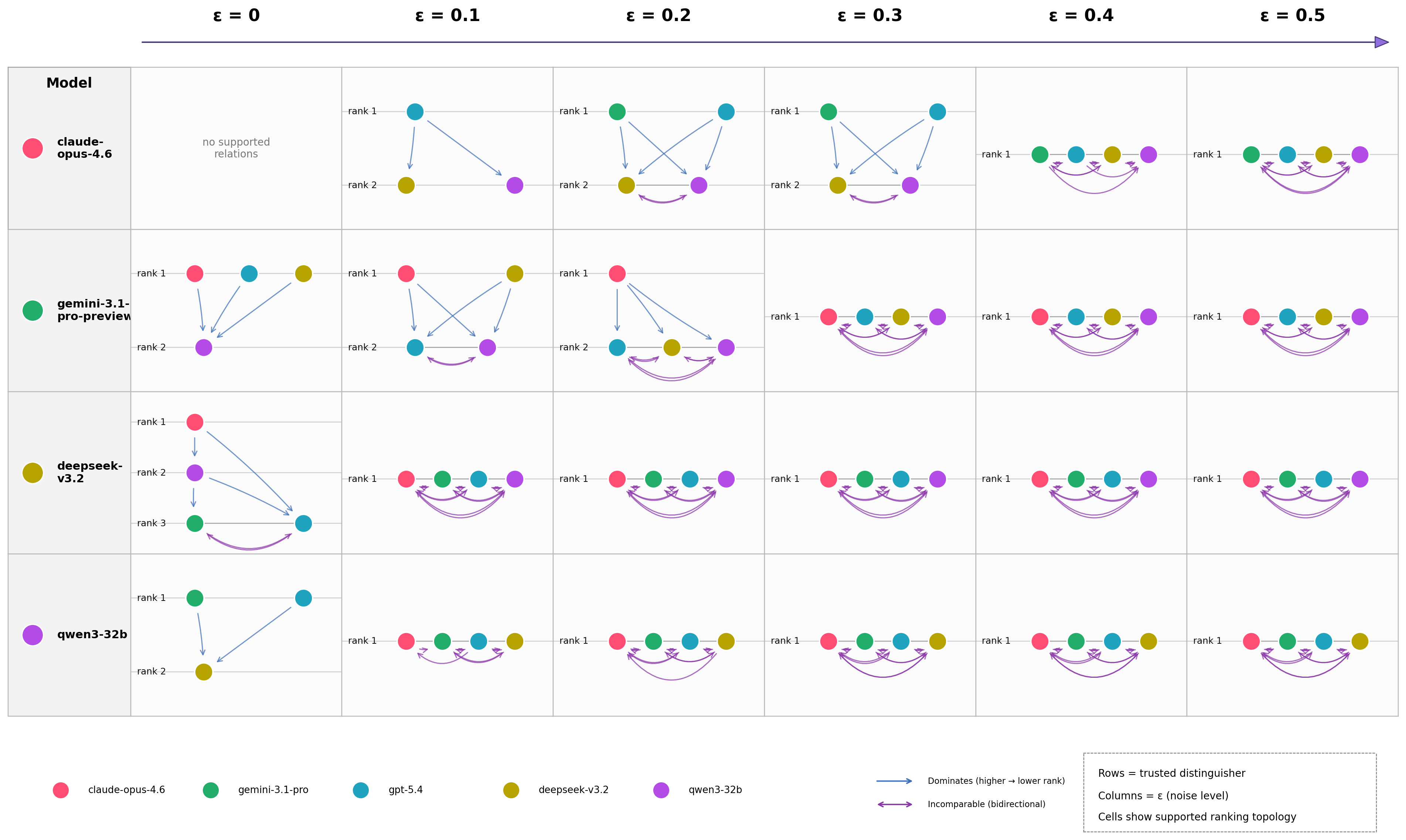}
\caption{Directed graphs where $A \to B$ if $A\succeq_{\epsilon; D} B\Longleftrightarrow \widehat d(A,B)\le \epsilon$. Models without incoming or outgoing edges are omitted.}
\label{fig:FD-graph-composite}
\end{figure}


\end{document}